%% file: arxiv.tex
\documentclass{article} 
\usepackage{nips14submit_e,times}
\usepackage{url}
\usepackage{amsmath}
\usepackage{color}
\usepackage{verbatim}
\usepackage{times}
\usepackage{natbib}
\bibpunct[; ]{(}{)}{,}{a}{}{;}
\usepackage{amssymb}
\def\slantfrac#1#2{\kern.1em^{#1}\kern-.3em/\kern-.1em_{#2}}

\usepackage{enumitem}

\usepackage{amsthm}
\newtheorem{theorem}{Theorem}
\newtheorem{lemma}[theorem]{Lemma}
\newtheorem{corollary}[theorem]{Corollary}
\newtheorem{proposition}[theorem]{Proposition}

\theoremstyle{remark}

\theoremstyle{definition}
\newtheorem{example}[theorem]{Example}
\newtheorem{definition}[theorem]{Definition}

\usepackage{graphicx}
\graphicspath{{./Figures/}}

\input{stuff} 

\usepackage{bibentry} 

\usepackage{xr}

\newcommand{\todog}[1]{#1}
\newcommand{\todor}[1]{\textcolor{black}{#1}}

\begin{document}

\nipsfinalcopy 
\title{On the Number of Linear Regions of\\ Deep Neural Networks}

\author{
Guido Mont\'ufar \\
Max Planck Institute for Mathematics in the Sciences\\ 
\texttt{montufar@mis.mpg.de} 
\And
{\bf Razvan Pascanu} \\
Universit\'{e} de Montr\'{e}al \\
\texttt{pascanur@iro.umontreal.ca}  \\
\and
{\bf Kyunghyun Cho}  \\
Universit\'{e} de Montr\'{e}al \\
\texttt{kyunghyun.cho@umontreal.ca}  \\
\And
{\bf Yoshua Bengio} \\
Universit\'{e} de Montr\'{e}al, CIFAR Fellow \\
\texttt{yoshua.bengio@umontreal.ca} 
}
\maketitle

\input{main}

\newpage
\appendix
\input{main_supp}

\end{document}

%% file: stuff.tex
\usepackage{sidecap}
\usepackage{mathtools}
\usepackage{wrapfig}

\newcommand{\bmx}[0]{\begin{bmatrix}}
\newcommand{\emx}[0]{\end{bmatrix}}

\newcommand{\qlay}[1]{#1}
\newcommand{\vect}[1]{\mathbf{#1}}

\newcommand{\matr}[1]{\mathbf{#1}}

\newcommand{\diag}[0]{\operatorname{diag}}

\newcommand{\vb}[0]{\vect{b}}

\newcommand{\vx}[0]{\vect{x}}
\newcommand{\vw}[0]{\vect{w}}

\newcommand{\mW}[0]{\matr{W}}

\newcommand{\NN}[0]{\mathcal{N}}
\newcommand{\Scal}[0]{\mathcal{S}}

\newcommand{\R}{\mathbb{R}}

%% file: main.tex
\begin{abstract}
We study the complexity of functions computable by deep feedforward neural networks with piecewise linear activations in terms of the symmetries and the number of linear regions that they have. Deep networks are able to sequentially map portions of each layer's input-space to the same output. In this way, deep models compute functions that react equally to complicated patterns of different inputs. The compositional structure of these functions enables them to re-use pieces of computation exponentially often in terms of the network's depth. This paper investigates the complexity of such compositional maps and contributes new theoretical results regarding the advantage of depth for neural networks with piecewise linear activation functions. In particular, our analysis is not specific to a single family of models, and as an example, we employ it for rectifier and maxout networks. We improve complexity bounds from pre-existing work and investigate the behavior of units in higher layers.  

\end{abstract}

\section{Introduction}
Artificial neural networks with several hidden layers, called \textit{deep}
neural networks, have become popular due to their unprecedented success in a
variety of machine learning tasks~\citep[see,
e.g.,][]{Krizhevsky-2012,Ciresan-et-al-2012,Goodfellow_maxout_2013,Hinton-et-al-2012}.
In view of this empirical evidence, deep neural networks are becoming
increasingly favoured over {\em shallow} networks (i.e., with a single layer of
hidden units), and are often implemented with more than five layers.  At the
time being, however, only a limited amount of publications have investigated
deep networks from a theoretical perspective. 
Recently, \citet{Delalleau+Bengio-2011-small} showed that a shallow network
requires exponentially many more sum-product hidden units\footnote{A single sum-product hidden layer summarizes a layer of product
units followed by a layer of sum units.} than a deep sum-product network in
order to compute certain families of polynomials.  We are interested in
extending this kind of analysis to more popular neural networks. 

There is a wealth of literature discussing approximation, estimation, and
complexity of artificial neural networks~\citep[see,
e.g.,][]{anthony2009neural}.  A well-known result states that a feedforward
neural network with a single, huge, hidden layer is a universal approximator of Borel
measurable functions~\citep[see][]{Hornik89,cybenko1989}. 
Other works have investigated universal approximation of probability distributions by deep
belief networks~\citep{Roux2010,Montufar:2011}, as well as their approximation
properties~\citep{montufar2013universal,krause13}.

These previous theoretical results, however, 
do not trivially apply to 
the types of deep neural networks that have seen success in recent years. 
Conventional neural networks often employ either hidden units with a bounded
smooth activation function, or Boolean hidden units. 
On the other hand, recently it has become more common to use piecewise linear functions, 
such as the {\em rectifier} activation $g(a)=\max\{0,a\}$ 
\mbox{\citep{Glorot+al-AI-2011-small,Nair-2010}} 
or the {\em maxout} activation $g(a_1,\ldots, a_k)=\max\{a_1,\ldots,a_k\}$ \mbox{\citep{Goodfellow_maxout_2013}}. 
The practical success of deep neural networks with piecewise linear units calls for the theoretical analysis specific for this
type of neural networks. 

In this respect, \citet{Pascanu2014} reported a theoretical result on
the complexity of functions computable by deep feedforward networks with
rectifier units. They showed that, in the asymptotic limit of many hidden
layers, deep networks are able to separate their input space into exponentially
more linear response regions than their shallow counterparts, 
despite using the same number of computational units. 

\todog{Building on the ideas from~\citep{Pascanu2014}, we develop a general framework for analyzing deep models with piecewise linear activations.} 
The intermediary layers of these models are able to map several pieces of their inputs into the same output. 
The layer-wise composition of the functions computed in this way re-uses low-level computations exponentially often as the number of layers increases. 
\todog{This key property enables deep networks to compute highly complex and structured functions.} 
We underpin this idea by estimating the number of linear regions of functions computable by two important types of piecewise linear networks: 
with rectifier units and with maxout units. 

Our results for the complexity of deep rectifier networks yield a significant improvement {over} the previous
results on rectifier networks mentioned above, 
showing a favourable behavior of deep over shallow networks even {with} a moderate number of hidden layers. 
\todog{Our analysis of deep rectifier and maxout networks serves as plattform to study a broad variety of related networks, 
such as convolutional networks.}

The number of linear regions of the functions that can be computed by a
given model is a measure of the model's flexibility.  An example of this
is given in Fig.~\ref{fig:eye_candy}, which compares the learnt {decision}
boundary of a single-layer and a two-layer model with the same number of
hidden units 
(see details in Appendix~\ref{sec:sinusoidal}).
This
illustrates the advantage of depth; the deep model captures the desired
boundary more accurately, approximating it with a larger number of linear
pieces.

As noted earlier, deep networks are able to {\em identify} an exponential
number of input neighborhoods by mapping them to a common output of some
intermediary hidden layer. 
The computations carried out on the activations of this intermediary layer are replicated
many times, once in each of the identified neighborhoods. 
This allows the networks to compute very complex looking functions even when they are
defined with relatively few parameters. 

The number of parameters is an upper bound for the dimension of the set of
functions computable by a network, and a small number of parameters means that
the class of computable functions has a low dimension.  The set of functions
computable by a deep feedforward piecewise linear network, although low
dimensional, achieves exponential complexity by re-using and composing features
from layer to layer.

\begin{figure}
\centering
\begin{minipage}{0.52\textwidth}
\centering
\includegraphics[width=.9\columnwidth,clip=true, trim=10cm 5cm 7cm 5cm]{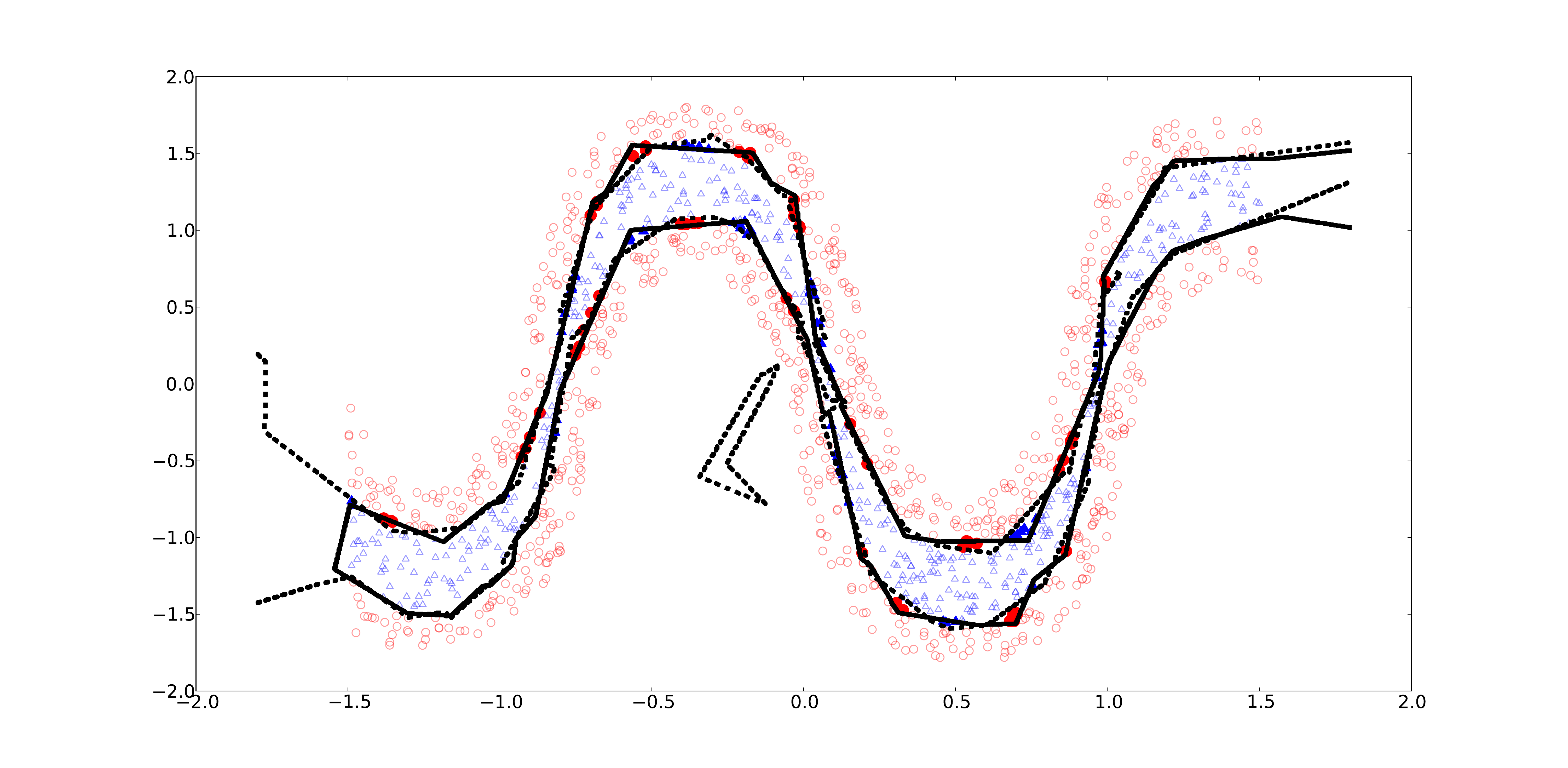}
\end{minipage}
\begin{minipage}{0.45\textwidth}
\centering
\includegraphics[width=.99\columnwidth,clip=true, trim=19cm 19cm 27cm 5.2cm]{single_layer.pdf}
\end{minipage}

\caption{Binary classification using a shallow model with 20 hidden units (solid line) and a deep model with two layers of 10 units each
    (dashed line). The right panel shows a close-up of the left panel. 
Filled markers indicate errors made by the shallow model.}
\label{fig:eye_candy}

\vspace{-3mm}
\end{figure}

\section{Feedforward Neural Networks and their Compositional Properties}
\label{sec:general}

In this section we discuss the ability of deep feedforward networks to re-map
their input-space to create complex symmetries by using only relatively few
computational units. The key observation of our analysis is that each layer of a
deep model is able to map different regions of its input to a common output.
This leads to a compositional structure, where computations on higher layers are
effectively replicated in all input regions that produced the same output at a
given layer.  The capacity to replicate computations over the input-space grows
exponentially with the number of network layers. 
Before expanding these ideas, we introduce basic definitions needed
in the rest of the paper.  At the end of this section, we give an intuitive
perspective for reasoning about the replicative capacity of deep models. 

\subsection{Definitions}
\label{sec:definitions}

A {\em feedforward neural network} is a composition of layers of computational units which
defines a function $F\colon \R^{n_0}\to\R^{\operatorname{out}}$  of the form  
\begin{equation}
    \label{eq:mlp_output}
F(\vx ; \theta) = 
f_{\operatorname{out}} \circ g_{\qlay{L}} \circ f_{\qlay{L}}\circ \cdots \circ g_{\qlay{1}} \circ f_{\qlay{1}} ( \vx)  ,
\end{equation}
where 
$f_l$ is a linear pre-activation function 
and $g_{\qlay{l}}$ is a nonlinear activation function. 
The parameter $\theta$ is composed of {\em input} weight matrices $\mW_{\qlay{l}}\in\R^{k\cdot n_l\times n_{l-1}}$ and {\em bias} vectors $\vb_l\in\R^{k \cdot n_l}$ for each layer $l\in[L]$. 

The output of the $l$-th layer is a vector 
$\vx_{\qlay{l}}  = \left[ \vx_{l,1} , \ldots, \vx_{l,n_l}  \right]^\top$ of activations $\vx_{l,i}$ of the units
$i\in[n_l]$ in that layer. 
This is computed from the activations of the preceding layer by $\vx_l=g_l(f_l (\vx_{l-1}))$. 
Given the activations $\vx_{l-1}$ of the units in the $(l-1)$-th layer, 
the pre-activation of layer $l$ is given by  
\[
f_{l} (\vx_{l-1}) = \mW_{l} \vx_{l-1} + \vb_{l}, 
\]
where $f_{\qlay{l}}  = \left[ f_{l,1} , \ldots, f_{l,{n_l}} \right]^\top$ is an array composed of $n_l$ pre-activation vectors $f_{l,i}\in\R^{k}$, and the activation of the $i$-th unit in the $l$-th layer is given by 
\[
  \vx_{l,i} = g_{l,i} (  f_{l,i} (\vx_{l-1}) ) .  
\]
We will abbreviate $g_l\circ f_l$ by $h_l$. 
When the layer index $l$ is clear, we will drop the corresponding subscript. 
We are interested in piecewise linear activations, 
and will consider the following two important types. 
\begin{itemize} 
    \item 
    \begin{minipage}{.21\textwidth} 
    Rectifier unit: 
    \end{minipage}
    $g_i (f_i ) = \max\left\{ 0, f_i  \right\}$, where $f_i \in \R$ and $k=1$.
    \item  
    \begin{minipage}{.21\textwidth} 
    Rank-$k$ maxout unit: 
    \end{minipage}
    		\begin{minipage}[t]{.7\textwidth}
    		$g_i(f_i) =  
 			\max\{ f_{i,1},\ldots, f_{i,k} \}$,  where $f_i=[f_{i,1},\ldots, f_{i,k}] \in \R^k$.  
            \end{minipage}
\end{itemize}

The {\em structure} of the network refers to the way its units are arranged. 
It is specified by the number $n_0$ of input dimensions, the number of layers $L$, and the number of units or {\em width} $n_l$ of each layer. 

We will classify the functions computed by different network structures, for different choices of parameters, in terms of their number of linear regions. 
A {\em linear region} of a piecewise linear function $F\colon \R^{n_0}\to\R^m$ is a maximal connected subset of the input-space $\R^{n_0}$, on which
$F$ is linear. 
For the functions that we consider, each linear region has full dimension, $n_0$. 

\subsection{Shallow Neural Networks}
\label{sec:single}

Rectifier units have two types of behavior; they can be either constant $0$ or linear, depending on their inputs. 
The boundary between these two behaviors is given by a hyperplane, 
and the collection of all the hyperplanes coming from all units in a rectifier layer forms a \emph{hyperplane arrangement}. 
In general, if the activation function $g\colon \R\to\R$ has a distinguished (i.e., irregular) behavior 
at
zero (e.g., an inflection point or non-linearity), then the function
$\R^{n_{0}}\to\R^{n_1};\; \vx \mapsto g(\mW \vx + \vb)$ has a distinguished
behavior at all inputs from any of the hyperplanes $H_i :=
\{\vx\in\R^{n_{0}}\colon \mW_{i,:} \vx + \vb_i=0\}$ for $i\in[n_1]$.  The hyperplanes
capturing this distinguished behavior also form a hyperplane arrangement. 

The hyperplanes in the arrangement split the input-space into several regions.
Formally, a {\em region} of a
hyperplane arrangement $\{H_1,\ldots, H_{n_1}\}$ is a connected component of
the complement $\R^{n_{0}} \setminus (\cup_i H_i)$, i.e., a set of points delimited
by these hyperplanes (possibly open towards infinity). 
The number of regions of an arrangement can be given in terms of a characteristic function of the arrangement, 
as shown in a well-known result by~\citet{zaslavsky1975facing}.  An arrangement of $n_1$
hyperplanes in $\R^{n_0}$ has at most $\sum_{j=0}^{n_0}{n_1\choose j}$ regions.
Furthermore, this number of regions is attained if and only if the hyperplanes
are in general position. 
This implies that the maximal number of linear regions of functions computed by a
shallow rectifier network with $n_0$ inputs and $n_1$ hidden units is $\sum_{j=0}^{n_0}{n_1\choose j}$ 
\citep[see][Proposition~5]{Pascanu2014}. 

\subsection{Deep Neural Networks}
\label{sec:deep}

We start by defining the identification of input neighborhoods mentioned in the introduction more formally:  

\begin{definition}
    \label{def:equivset}
     A map $F$ {\em identifies} two neighborhoods $S$ and $T$ of its input domain if it maps them to a common subset $F(S) = F(T)$ of its output domain. 
In this case we also say that $S$ and $T$ are {\em identified} by $F$. 
\end{definition}

For example, the four quadrants of 2-D Euclidean space
are regions that are  identified by the absolute value function
$g\colon \R^2\to\R^2$; 
\begin{equation}
     \label{eq:abs}
    g(x_1,x_2) = \left[ |x_1| \;,\; |x_2| \right]^\top. 
\end{equation}

The computation carried out by the $l$-th layer of a feedforward network on a
set of activations from the $(l-1)$-th layer is effectively carried out for all
regions of the input space that lead to the same activations of the $(l-1)$-th
layer.  One can choose the input weights and biases of a given layer in such a
way that the computed function behaves most interestingly on those activation
values of the preceding layer which have the largest number of preimages in the
input space, thus replicating the interesting computation many times in the input
space and generating an overall complicated-looking function. 

For any given choice of the network parameters, each hidden layer $l$ computes
a function $h_l = g_l \circ f_l$ on the output activations of the preceding
layer.  We consider the function $F_l\colon \R^{n_0}\to \R^{n_l}$; $F_l  :=
h_l\circ \cdots \circ h_1$ that computes the activations of the $l$-th hidden
layer.  We denote the image of $F_l$ by $S_l\subseteq \R^{n_l}$, i.e., the set
of (vector valued) activations reachable by the $l$-th layer for all possible
inputs. 
Given a subset $R \subseteq S_l$, we denote by $P^l_R$ the set of subsets 
$\bar{R}_1, \ldots, \bar{R}_k \subseteq S_{l-1}$ that are mapped by $h_l$ onto $R$; that is, subsets that satisfy 
$h_l(\bar{R}_1)=\cdots = h_l (\bar{R}_k) = R$. 
See Fig.~\ref{fig:equivalence} for an illustration. 

The number of separate input-space neighborhoods  
that are mapped to a common neighborhood $R\subseteq S_l\subseteq \R^{n_l}$ 
can be given recursively as 
\begin{align}
    \label{eq:ncopies} 
    \NN_R^{l} =  \sum_{R' \in  P^l_R} \NN_{R'}^{l-1}, \qquad 
    \NN_R^{0} =  1,\mbox{ for each region } R \subseteq \R^{n_0}. 
\end{align}
For example, $ P^1_R$ is the
set of all disjoint input-space neighborhoods 
whose image by the function computed by the first layer, $h_1\colon \vx\mapsto g (\mW \vx + \vb )$, equals $R \subseteq S_1\subseteq\R^{n_1}$. 

The recursive formula~\eqref{eq:ncopies} counts the number of identified sets by moving along the branches of a tree rooted at the set $R$ of the $j$-th layer's output-space  
(see Fig.~\ref{fig:equivalence}~(c)).  
Based on these observations, 
we can estimate the maximal number of linear regions as follows. 

\begin{figure}[t]
\centering
    \begin{minipage}[b]{0.48\textwidth}
        \centering
        \includegraphics[width=0.8\columnwidth]{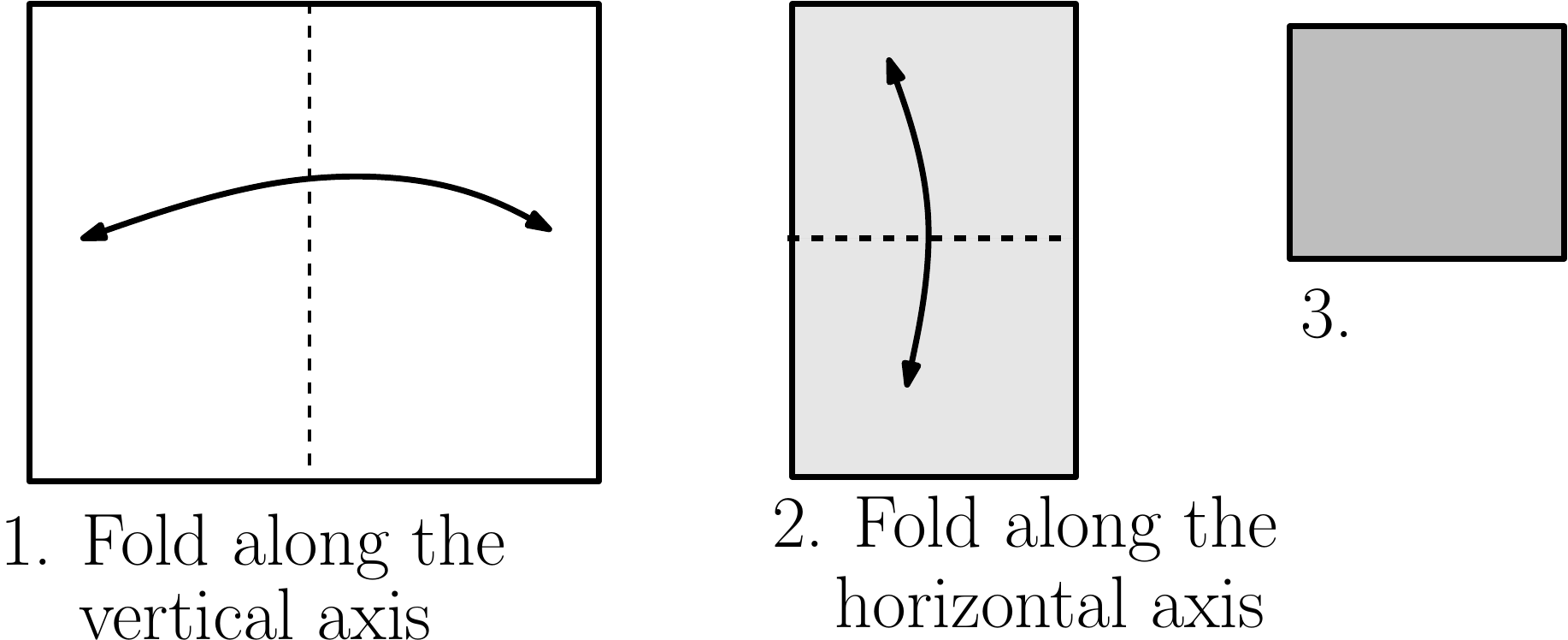}

\vspace{.2cm}

        (a)

\vspace{.3cm}

        \includegraphics[width=0.8\columnwidth]{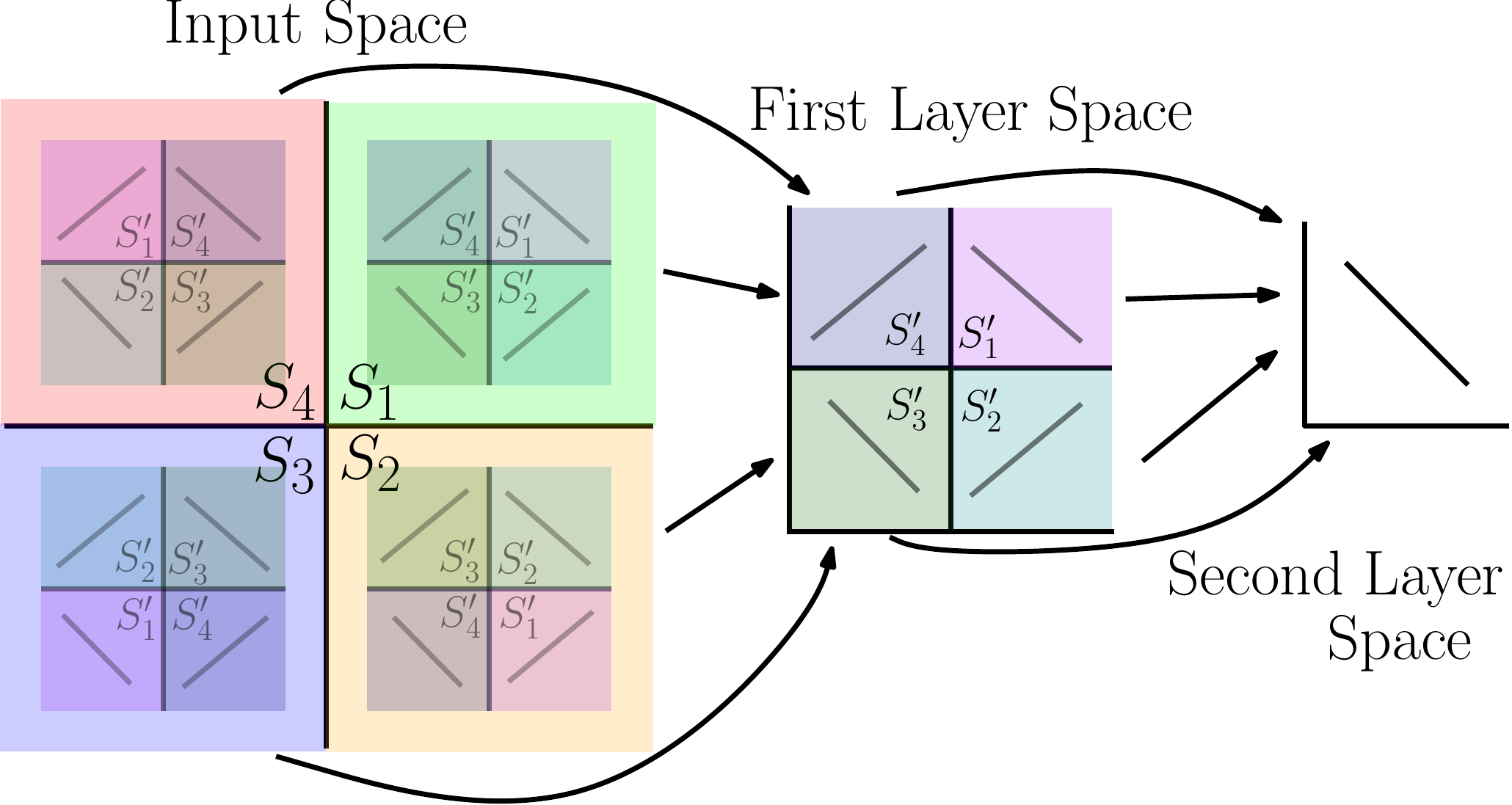}

        (b)
    \end{minipage}
    \qquad
    \begin{minipage}[b]{0.4\textwidth}
    \centering
    \includegraphics[width=.9\textwidth]{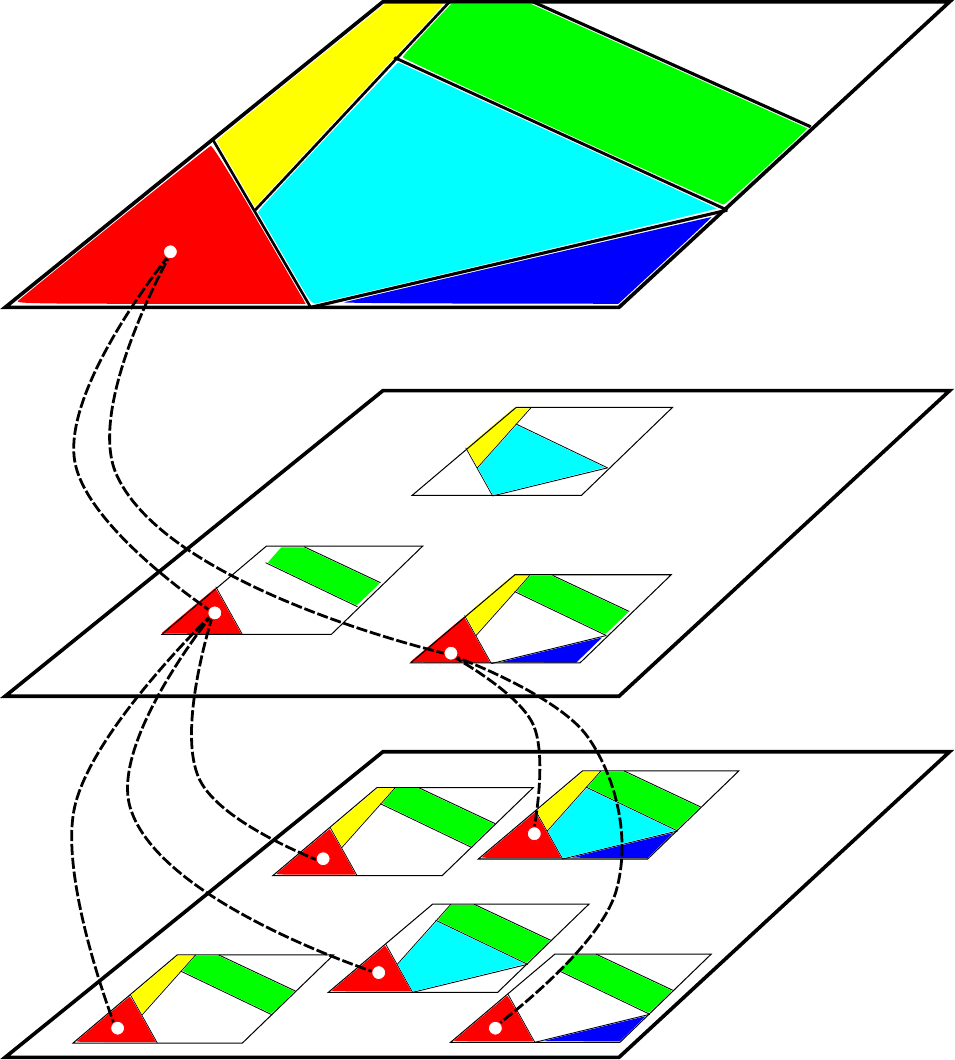}
    
    \vspace{.2cm}
    
    (c)
    \end{minipage}
    \caption{
    (a) Space folding of 2-D Euclidean space along the
    two axes. 
    (b) An illustration of how the top-level
partitioning (on the right) is replicated to the original input space (left). \label{fig:paper_folding}
    (c) Identification of regions across the layers of a deep model. \label{fig:equivalence}
    }
\vspace{-3mm}
\end{figure}

\begin{lemma} 
    \label{thm:gen1}
    The maximal number of linear regions of the functions computed by an $L$-layer neural
    network with 
    piecewise linear activations is at least 
$ \NN =  \sum_{R\in P^L} \NN_R^{L-1} $, 
    where $\NN_R^{L-1}$ is defined by Eq.~\eqref{eq:ncopies}, 
    and $P^L$ is a set of neighbordhoods in distinct linear regions of the function computed by the last hidden layer. 
\end{lemma}

Here, the idea to construct a function with many linear regions is to use the first $L-1$ hidden layers to identify many input-space neighborhoods, mapping all of them to the activation neighborhoods $P^L$ of the $(L-1)$-th hidden layer, 
each of which belongs to a distinct linear region of the last hidden layer. 

We will give the detailed analysis of rectifier and maxout networks in Secs.~\ref{sec:rectifier} and~\ref{sec:maxout}. 

\subsection{Identification of Inputs as Space Foldings}

In this section, we discuss an intuition behind Lemma~\ref{thm:gen1} in terms
of {\em space folding}. 
A map $F$ that identifies two subsets $\Scal$ and
$\Scal'$ can be considered as an operator that {\em folds} its domain in such a way that the two subsets $\Scal$ and $\Scal'$ coincide and
are mapped to the same output. 
For instance, the absolute value function $g\colon \R^2\to \R^2$ from Eq.~\eqref{eq:abs} folds its domain twice (once along each coordinate axis), as illustrated in Fig.~\ref{fig:equivalence}~(a). 
This folding identifies the four quadrants of 2-D
Euclidean space. 
By composing such operations, the same 
kind of map can be applied again to the output, in order to re-fold the first folding. 

Each hidden layer of a deep neural network can be associated with a folding operator. 
Each hidden layer folds the space of activations of the previous layer. 
In turn, a deep neural network effectively folds its input-space recursively, starting with the first layer. 
The consequence of this recursive folding is that any function computed on the final folded space 
will apply to all the collapsed subsets identified by the map corresponding to the succession of foldings. 
This means that in a deep model any partitioning of the last layer's image-space is replicated in 
all input-space regions which are identified by the succession of foldings. 
Fig.~\ref{fig:paper_folding}~(b) offers an illustration of this replication property. 

\begin{figure}[t]
    \centering
    \includegraphics[width=0.6\columnwidth]{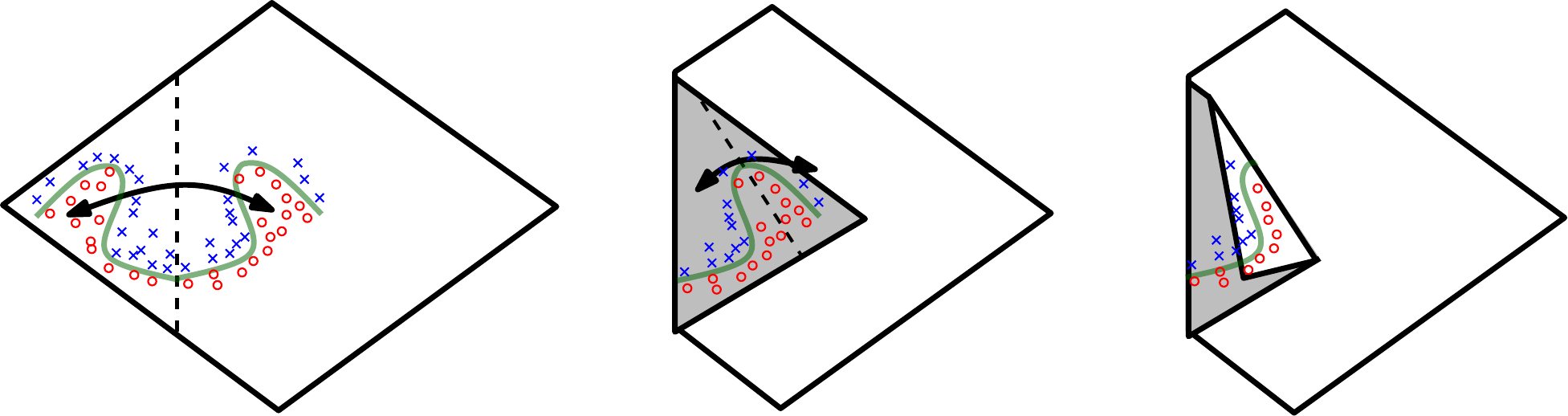}
    \caption{Space folding of 2-D space in a non-trivial way. Note how the folding 
    can potentially identify symmetries in the boundary that it needs to learn.}
    \label{fig:paper_folding_complex}
\vspace{-3mm}
\end{figure}

Space foldings are not restricted to foldings along coordinate axes and they do
not have to preserve lengths.  Instead, the space is folded depending on the
orientations and shifts encoded in the input weights $\mW$ and biases $\vb$ and
on the nonlinear activation function used at each hidden layer. 
In particular, this means that the sizes and orientations of identified
input-space regions may differ from each other.  See
Fig.~\ref{fig:paper_folding_complex}.

\subsection{Stability to Perturbation}

Our bounds on the complexity attainable by deep models (Secs.~\ref{sec:rectifier} and~\ref{sec:maxout}) are based on suitable
choices of the network weights.  However, this does not mean that the indicated
complexity is only attainable in singular cases. 

The parametrization of the functions computed by a neural network is
continuous.  More precisely, the map $\psi \colon \R^N \to C(\R^{n_0};
\R^{n_L})$; $\theta \mapsto F_\theta$, which maps input weights and biases
$\theta = \{\mW_i, \vb_i\}_{i=1}^L$ to the continuous functions $F_\theta\colon
\R^{n_0}\to \R^{n_L}$ computed by the network,
is continuous. 
Our analysis considers the number of linear regions of the functions $F_\theta$. 
By definition, each linear region contains an open neighborhood of the
input-space $\R^{n_0}$. 
Given any function $F_\theta$ with a finite number of linear regions, 
there is an $\epsilon > 0$ such that for each
$\epsilon$-perturbation of the parameter $\theta$, the resulting function
$F_{\theta+\epsilon}$ has at least as many linear regions as $F_\theta$. 
The linear regions of $F_\theta$ are preserved under small perturbations
of the parameters, because they have a finite volume. 

If we define a probability density on the space of parameters, 
what is the probability of the event that the function represented by the
network has a given number of linear regions?  By the above discussion, the
probability of getting a number of regions at least as large as the number
resulting from any particular choice of parameters (for a uniform measure
within a bounded domain) is nonzero, even though it may be very small. This is
because there exists an epsilon-ball of non-zero volume around that particular
choice of parameters, for which at least the same number of linear regions is
attained. 

For future work it would be interesting to study the partitions of parameter
space $\R^N$ into pieces where the resulting functions partition their
input-spaces into isomorphic linear regions, and to investigate how many of
these pieces of parameter space correspond to functions with a given number of
linear regions. 

\subsection{Empirical Evaluation of Folding in Rectifier MLPs}
\label{sec:faces}

We empirically examined the behavior of a trained MLP to see if it folds the
input-space in the way described above.  First, we make the observation that
tracing the activation of each hidden unit in this model gives a piecewise
linear map $\R^{n_0}\to \R$ (from inputs to activation values of that unit).
Hence, we can analyze the behavior of each unit by visualizing the {different 
weight matrices corresponding to the different linear pieces of this map.} 
The weight matrix of one piece of this map can be found by tracking the linear
piece used in each intermediary layer, starting from an input example.
This visualization technique, a byproduct of our theoretical analysis, is similar to the one proposed
by~\citet{Zeiler+et+al-arxiv2013b}, but is motivated by a different perspective.

After computing the activations of an intermediary hidden unit for each
training example, we can, for instance, inspect two examples that result in
similar levels of activation for a hidden unit. With the linear maps of the
hidden unit corresponding to the two examples we perturb one of the examples 
until it results in exactly the same activation. These two inputs then can be
safely considered as points in two regions identified by the hidden unit. 
In 
Appendix~\ref{sec:visualization}
we provide details and examples of this
visualization technique. We also show inputs identified by a deep MLP. 

\section{Deep Rectifier Networks}
\label{sec:rectifier}

In this section we analyze deep neural networks with rectifier units, based on
the general observations from Sec.~\ref{sec:general}.  We improve upon the
results by \citet{Pascanu2014}, with a tighter lower-bound on the maximal
number of linear regions of functions computable by deep rectifier networks. 

First, let us note the following upper-bound, which follows directly from the
fact that each linear region of a rectifier network corresponds to a pattern of
hidden units being active: 
\begin{proposition}
The maximal number of linear regions of the functions computed by any rectifier
network with a total of $N$ hidden units is bounded from above by $2^N$. 
\end{proposition}

\subsection{Illustration of the Construction}
\label{sec:illustration}

\begin{figure}[t]
    \centering
    \includegraphics[width=0.4\columnwidth]{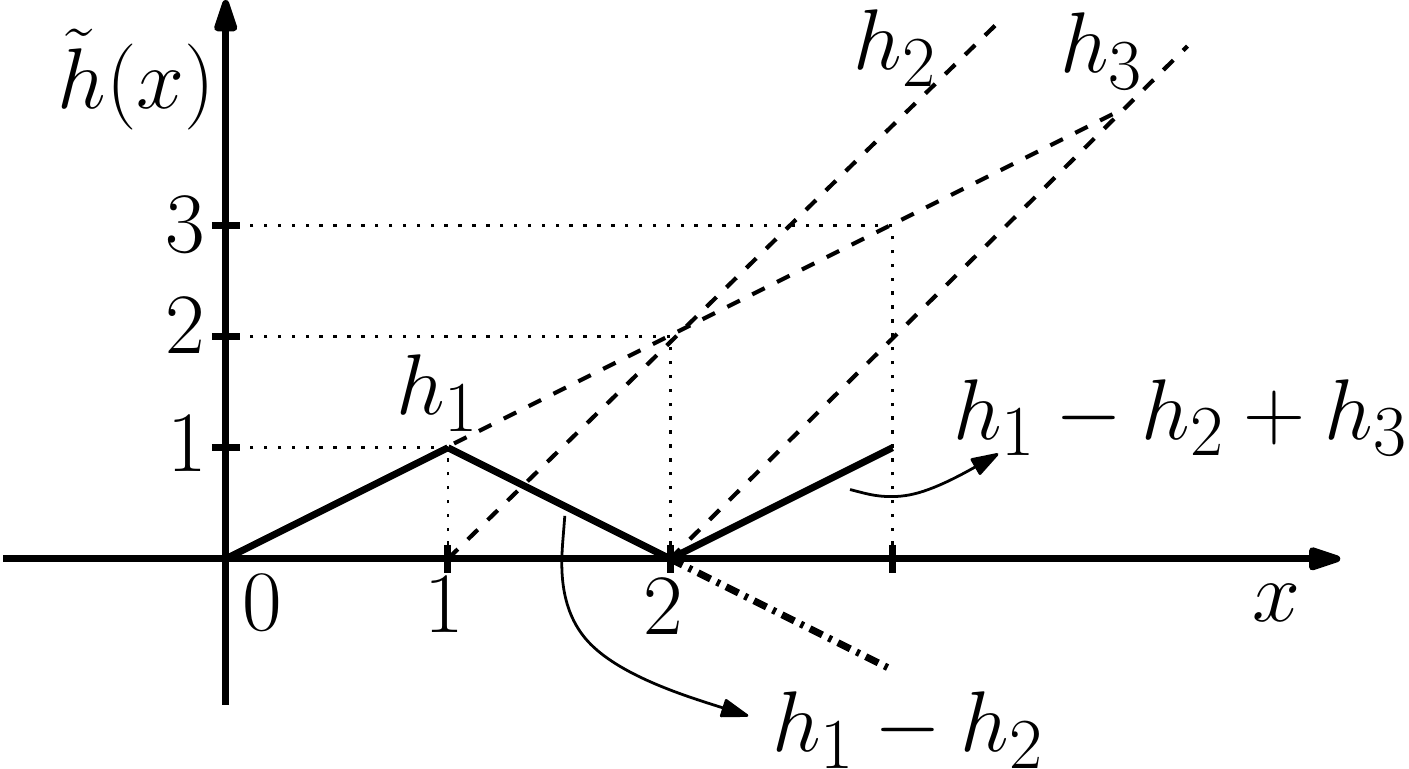}
    \caption{Folding of the real line into equal-length segments
        by a sum of rectifiers.}
    \label{fig:simple}
\vspace{-3mm}
\end{figure}

Consider a layer of $n$ rectifiers with $n_0$ input variables, where $n \geq n_0$.  
We partition the set of rectifier units into $n_0$ (non-overlapping) subsets of cardinality $p =
\left\lfloor\slantfrac{n}{n_0}\right\rfloor$ and ignore the remainder units. 
Consider the units in the $j$-th subset. 
We can choose their input weights and biases such that
\begin{eqnarray*}
    h_1 (\vx)\!\! &=& \!\!\max\left\{0, \phantom{2}\vw \vx \right\}, \\
    h_2 (\vx)\!\! &=& \!\!\max\left\{0, 2 \vw \vx  - 1 \right\}, \\
    h_3 (\vx)\!\! &=& \!\!\max\left\{0, 2 \vw \vx  - 2 \right\}, \\
    		  &\vdots& \\
    h_p (\vx)\!\! &=& \!\!\max\left\{0, 2 \vw \vx  - (p - 1)\right\},
\end{eqnarray*}
where $\vw$ is a row vector with $j$-th entry equal to $1$ and all other
entries set to $0$.  The product $\vw \vx$ selects the $j$-th coordinate of
$\vx$.
Adding these rectifiers with alternating signs, we obtain following scalar function: 
\begin{align}
    \label{eq:comb2}
    \tilde h_j(\vx) = 
    \left[ 1,  -1 , 1 , \ldots, (-1)^{p-1}  \right]
    \left[ h_1 (\vx) , h_2 (\vx) , h_3 (\vx), \ldots, h_p(\vx)\right]^\top. 
\end{align}
Since $\tilde h_j$ acts only on the $j$-th input coordinate, we
may redefine it to take a scalar input, namely the $j$-th coordinate of $\vx$.
This function has $p$ linear regions given by the following intervals: 
\begin{align*}
\left(-\infty, 0\right] , 
        \left[0, 1 \right] , 
        \left[1, 2 \right] , 
        \ldots , 
        \left[p-1, \infty \right).
\end{align*}
Each of these intervals has a subset that is mapped by $\tilde h_j$ onto the
interval $\left( 0,1\right)$, as illustrated in Fig.~\ref{fig:simple}.  The function
$\tilde h_j$ identifies the input-space strips with $j$-th coordinate $\vx_j$
restricted to the intervals $(0,1), (1,2), \ldots, (p-1,p)$. 
Consider now all the $n_0$ subsets of rectifiers and the function $\tilde h =
\big[ \tilde h_1, \tilde h_2, \ldots, \tilde h_p \big]^\top$.  This function is
locally symmetric about each hyperplane 
with a fixed $j$-th coordinate equal to 
$\vx_j = 1,\ldots, \vx_j=p-1$ (vertical lines in Fig.~\ref{fig:simple}), for all $j=1,\ldots, n_0$. 
Note the periodic pattern that emerges. 
In fact, the function $\tilde h$ identifies a total of $p^{n_0}$ hypercubes
delimited by these hyperplanes. 

Now, note that $\tilde h$ arises from $h$ by composition with a linear function (alternating sums). 
This linear function can be effectively absorbed in the pre-activation function of the next layer. 
Hence we can treat $\tilde h$ as being the function computed by the current layer. 
Computations by deeper layers, as functions of the unit hypercube output of this rectifier layer, 
are replicated on each of the $p^{n_0}$ identified input-space hypercubes.

\subsection{Formal Result}

We can generalize the construction described above to the case of a deep
rectifier network with  $n_0$ inputs and $L$ hidden layers of widths $n_i\geq
n_0$ for all $i\in\left[L\right]$.  We obtain the following lower bound for the
maximal number of linear regions of deep rectifier networks: 

\begin{theorem} 
    \label{thm:opt2}
    The maximal number of linear regions of the functions computed by a
    neural network with $n_0$ input units and $L$ hidden layers, with $n_i\geq
    n_0$ rectifiers at the $i$-th layer, is lower bounded by 
    \[
    \left( \prod_{i=1}^{L-1} \left\lfloor \frac{n_i}{n_0}
            \right\rfloor^{n_0} \right)
            \sum_{j=0}^{n_0} { n_L \choose j }. 
    \]
\end{theorem}

The next corollary gives an expression for the asymptotic behavior of these bounds. 
Assuming that $n_0 = O(1)$ and $n_i = n$ for all $i \geq 1$, the number of
regions of a single layer model with $Ln$ hidden units behaves as
$O(L^{n_0}n^{n_0})$~\citep[see][Proposition~10]{Pascanu2014}.  For a deep
model, Theorem~\ref{thm:opt2} implies: 

\begin{corollary}
    \label{remark:asymptotic_mlp}
A rectifier neural network with $n_0$ input units and $L$ hidden layers of
width $n \geq n_0$ can compute functions that have 
$\Omega\left(\left(\slantfrac{n}{n_0}\right)^{(L-1)n_0} n^{n_0}\right)$ linear regions. 
\end{corollary}

Thus we see that the number of linear regions of deep models grows
exponentially in $L$ and polynomially in $n$, which is much faster than that of
shallow models with $nL$ hidden units.  Our result is a significant improvement
over the bound  $\Omega\left( \left( \slantfrac{n}{n_0}\right)^{L-1}n^{n_0}\right)$
obtained by~\citet[]{Pascanu2014}.  In particular, our result demonstrates that
even for small values of $L$ and $n$, deep rectifier models are able to produce
substantially more linear regions than shallow rectifier models. 
Additionally, using the same strategy as \citet[]{Pascanu2014}, 
our result can be reformulated in terms of the number of \emph{linear
regions per parameter}.  This results in a similar behaviour, with
deep models being exponentially more efficient than shallow models
(see Appendix~\ref{sec:bound_params}).  
\section{Deep Maxout Networks}
\label{sec:maxout}

A maxout network is a feedforward network with layers defined as follows: 

\begin{definition}
    \label{def:maxout}
A {\em rank-$k$ maxout layer} with $n$ input and $m$ output units is defined by a pre-activation function of the form 
$f:\R^n\to \R^{m\cdot k};\,f(\mathbf{x}) = \mathbf{W} \mathbf{x} + \mathbf{b}$,  
with input and bias weights $\mathbf{W}\in\R^{m\cdot k \times n}$, $\mathbf{b}\in\R^{m \cdot k}$, 
and activations of the form 
$g_j(\mathbf{z}) = \max\{ \mathbf{z}_{(j-1)k+1}, \ldots, \mathbf{z}_{jk}\}$ for all $j\in[m]$. 
The  layer computes a function 
\begin{equation}
g\circ f\colon \quad 
\R^n\to\R^m; 
\quad
\vx \mapsto 
\begin{pmatrix}
\max\{f_1(\vx), \ldots, f_k(\vx)  \} \\ \vdots \\ 
\max\{f_{(m-1) k +1}(\vx), \ldots, f_{m k}(\vx)  \}
\end{pmatrix}. 
\end{equation}

\end{definition}

Since the maximum of two convex functions is convex,  
maxout units 
and maxout layers 
compute convex functions. 
The maximum of a collection of functions is called their {\em upper envelope}. 
We can view the graph of each linear function $f_i \colon \R^n\to \R$ as a supporting hyperplane of a convex set in
$(n+1)$-dimensional space. 
In particular, if each $f_i$, $i\in [k]$ is the unique maximizer 
$f_i = \max\{f_i'\colon i'\in[k] \}$ at some input neighborhood,  
then the number of linear regions of the upper envelope $g_1\circ f =\max\{f_i\colon i\in[k] \}$ is exactly $k$. 
This shows that the maximal number of linear regions of a maxout unit is equal to its rank. 

The linear regions of the maxout layer 
are the intersections of the linear regions of the individual maxout units. 
In order to obtain the number of linear regions for the layer, 
we need to describe the structure of the linear regions of each maxout unit, and study their possible intersections. 

Voronoi diagrams can be lifted to upper envelopes of linear functions, and hence they describe input-space partitions generated by maxout units. 
Now, how many regions do we obtain by intersecting the regions of $m$ Voronoi diagrams with $k$ regions each? 
Computing the intersections of Voronoi diagrams is not easy, in general. 
A trivial upper bound for the number of linear regions is $k^m$, which corresponds to the case where all intersections of regions of different units are different from each other. We will give a better bound in Proposition~\ref{prp:maxout}. 

Now, for the purpose of computing lower bounds, here it will be sufficient to consider certain well-behaved special cases. 
One simple example is 
the division of input-space by 
$k-1$ parallel hyperplanes. 
If $m\leq n$, we can consider the arrangement of hyperplanes $H_i= \{\vx \in\R^n \colon \vx_j=i\}$ for $i=1,\ldots, k-1$, for each
maxout unit $j\in [m]$. 
In this case, the number of regions is $k^m$. 
If $m> n$, the same arguments yield $k^n$ regions. 

\begin{proposition}
    \label{prp:maxout}
The maximal number of regions of a single layer maxout network
with $n$ inputs and $m$ outputs of rank $k$ is lower bounded by
$k^{\min\left\{n,m\right\}}$ and upper bounded by 
$\min \{  \sum_{j=0}^{n}{ k^2 m\choose j} , k^m \}$. 
\end{proposition}

Now we take a look at the deep maxout model. 
Note that a rank-$2$ maxout layer can be simulated by a rectifier layer with twice as many units.  Then, by the
results from the last section, a rank-$2$ maxout network with $L-1$ hidden
layers of width $n=n_0$ can identify $2^{n_0 (L-1) }$ input-space regions, and,
in turn, it can compute functions with $2^{n_0(L-1)} 2^{n_0} = 2^{n_0 L}$
linear regions. 

For the rank-$k$ case, we note that a rank-$k$ maxout unit can identify $k$
cones from its input-domain, whereby each cone is a neighborhood of the
positive half-ray $\{r \mW_i \in\R^n \colon r\in\R_+ \}$ corresponding to the
gradient $\mW_i$ of the linear function $f_i$ for all $i\in[k]$.  Elaborating
this observation, we obtain: 

\begin{theorem}
    \label{thm:maxout}
A maxout network with $L$ layers of width $n_0$ and rank $k$ can compute
functions with at least $k^{L-1} k^{n_0}$ linear regions. 
\end{theorem}

Theorem~\ref{thm:maxout} and Proposition~\ref{prp:maxout} show that deep maxout
networks can compute functions with a number of linear regions that grows
exponentially with the number of layers, and exponentially faster than the maximal number of regions of shallow models with the same number of units. 
Similarly to the rectifier model, 
this exponential behavior can also be established with respect to the number of network parameters. 

We note that although certain functions that can be computed by maxout
layers can also be computed by rectifier layers, the rectifier
construction from last section leads to functions that are not computable
by maxout networks (except in the rank-$2$ case).  The proof of
Theorem~\ref{thm:maxout} is based on the same general arguments from
Sec.~\ref{sec:general}, but uses a different construction than
Theorem~\ref{thm:opt2} 
(details in Appendix~\ref{sec:maxout_apdx}). 

\section{Conclusions and Outlook}

We studied the complexity of functions computable by deep feedforward neural
networks in terms of their number of linear regions. We specifically focused on
deep neural networks having piecewise linear hidden units which have been found
to provide superior performance in many machine learning applications recently.
We discussed the idea that each layer of a deep model is able to identify pieces
of its input in such a way that the composition of layers identifies an
exponential number of input regions. This results in exponentially replicating the
complexity of the functions computed in the higher layers of the model. The
functions computed in this way by deep models are complicated, but still they have an intrinsic rigidity 
caused by the replications, which may help deep models generalize to unseen samples better than shallow models. 

This framework is applicable to any neural network that has a piecewise linear
activation function. For example, if we consider a convolutional network with
rectifier units, as the one used in \citep{Krizhevsky-2012}, we can see that the
convolution followed by max pooling at each layer identifies all patches of the
input within a pooling region. This will let such a deep convolutional neural
network recursively identify patches of the images of lower layers, resulting in
exponentially many linear regions of the input space. 

The parameter space of a given network is partitioned into the regions where the
resulting functions have corresponding linear regions. This
correspondence of the linear regions of the computed functions can be
described in terms of their adjacency structure, or a poset of intersections of
regions.  Such combinatorial structures are in general hard to compute, even for
simple hyperplane arrangements. One interesting question for future analysis is
whether many regions of the parameter space of a given network correspond to
functions which have a given number of linear regions. 


\subsubsection*{References}
\pagestyle{empty}
 \begingroup
  \def\section*#1{}
\bibliographystyle{abbrvnat}
\bibliography{myref,strings,ml}
\endgroup

%% file: main_supp.tex
\section{Identification of Input-Space Neighborhoods}

\begin{proof}[Proof of Lemma~\ref{thm:gen1}]

Each output-space neighborhood $R\in P^L$ has as preimages all input-space neighborhoods that are $R$-identified by $\eta_L$ (i.e., the input-space neighborhoods whose image by $\eta_L$, the function computed by the first $L$-layers of the network, equals $R$). 
The number of input-space preimages of $R$ is denoted $\NN_R^{L}$. 
If each $R\in P^L$ is the image of a distinct linear region of the function $h_L = g_L\circ f_L$ computed by the last layer, 
then, by continuity, all preimages of all different $R\in P^L$ belong to different linear regions of $\eta_L$. 
Therefore, the number of linear regions of functions computed by the entire network is at least equal to the sum of the number of preimages of all $R\in P^L$, which is just $\NN = \sum_{R \in P^L} \NN_R^{L-1}$.  
\end{proof}

\section{Rectifier Networks}

\begin{proof}[Proof of Theorem~\ref{thm:opt2}]
The proof is done by counting the number of regions for a suitable choice of network parameters. 
The idea of the construction is to divide the first $L-1$ layers of the network into $n_0$ independent parts; 
one part for each input neuron. 
The parameters of each part are chosen in such a way that it folds its one-dimensional input-space many times into itself. 
For each part, the number of foldings per layer is equal to the number of units per layer. See Fig.~\ref{figure:maxout}. 

As outlined above, we organize the units of each layer into $n_0$ non-empty groups of units of sizes $p_1,\ldots, p_{n_0}$. 
A simple choice of these sizes is, for example, $p_1=\cdots= p_{n_0}= \left\lfloor n_l / n_0  \right\rfloor$ for a layer of width $n_l$, dropping the remainder units. 
We define the input weights of each group in such a way that the units in that group are sensitive to only one coordinate 
of the $n_0$-dimensional input-space. 
By the discussion from Sec.~\ref{sec:illustration}, choosing the input and bias weights in the right way, 
the alternating sum of the activations of the $p$ units within one group 
folds their input-space coordinate $p$ times into itself. 
Since the alternating sum of activations is an affine map, it can be absorbed in the pre-activation function of the next layer. 
In order to make the arguments more transparent, we view the alternating sum $\tilde h = h_1 - h_2 + \cdots \pm h_p $ 
of the activations of the $p$ units in a group as the activation of a fictitious intermediary unit. 
The compound output of all these $n_0$ intermediary units partitions the input-space, $\R^{n_0}$, 
into a grid of $\prod_{i=1}^{n_0} p_i$ identified regions. 
Each of these input-space regions is mapped to the $n_0$-dimensional unit cube in the output-space of the intermediary layer. 
We view this unit cube as the {\em effective} inputs for the next hidden layer, and repeat the construction. 
In this way, with each layer of the network, the number of identified regions is multiplied by $\prod_{i=1}^{n_0}p_i$, according to Lemma~\ref{thm:gen1}. 
In the following we discuss the folding details explicitly. 

Consider one of the network parts and consider the weights used in Sec.~\ref{sec:illustration}. 
The function $\tilde h$ computes an alternating sum of the responses $h_k$ for $k\in[p]$.  
It is sufficient to show that this sum folds the input-coordinate $p$ times into the interval $\left(0,1\right)$. 
Inspecting the values of $h_k$, we see that we only need to explore the intervals $(0,1), (1, 2), \ldots, (p-1, p)$. 

Consider one of these intervals, $(k-1,k)$, for some $k\in[p]$. 
Then, for all $x\in (k-1,k)$, 
we have $\tilde h (x) = x + 2 \sum_{i=1}^{k-1} (-1)^i (x-i) = (-1)^{k-1} (x-(k-1)) - \frac{1}{2}((-1)^{k-1} -1) $.  
Hence $\tilde h(k-1) = -\frac{1}{2}((-1)^{k-1} -1)$ and $\tilde h(k) = (-1)^{k-1} - \frac{1}{2}((-1)^{k-1} -1)$. 
One of the two values is always zero and the other one, and so  $\tilde h(\{k-1,k\}) = \{0,1\}$. 
Since the function is linear between $k-1$ and $k$, we obtain that $\tilde h$ maps the interval $(k-1,k)$ to the interval $(0,1)$.

In total, the number of input-space neighborhoods that are mapped by the first $L-1$ layers onto the (open) unit hypercube $(0,1)^{n_0}$ of the  (effective) output space of the $(L-1)$-th layer is given by 
\begin{equation}
\mathcal{N}^{L-1}_{(0,1)^{n_0}} = \prod_{l=1}^{L-1} \prod_{i=1}^{n_0} p_{l,i},  
\end{equation}
where $p_{l,i}$ is the number of units in the $i$-th group of units in the $l$-th layer. 

The inputs and bias of the last hidden layer can be chosen in such a way that the function $h_L$ partitions its (effective) input neighborhood 
$(0,1)^{n_0}$ by an arrangement of $n_L$ hyperplanes in general position, i.e., into $\sum_{j=0}^{n_0}{n_L\choose j}$ regions (see Sec.~\ref{sec:single}). 

Let $m_l$ denote the remainder of $n_l$ divided by $n_0$.  
Choosing $p_{l,1}=\cdots=p_{l,n_0-m_l} = \left\lfloor  {n_l} / {n_0}\right\rfloor$ 
and $p_{l,n_0-m_l+1}= \cdots = p_{l,n_0} = \left\lfloor  {n_l} / {n_0}\right\rfloor +1$, 
we obtain a total of linear regions 
\begin{equation}
\left( \prod_{l=1}^{L-1} \left\lfloor \frac{n_l}{n_0}\right\rfloor^{n_0 - m_l} \left( \left\lfloor \frac{n_l}{n_0}\right\rfloor +1\right)^{m_l} \right) \; \sum_{j=0}^{n_0}{n_L\choose j}.   
\end{equation}
This is equal to the bound given in theorem when all remainders $m_l= n_l - n_0 \lfloor n_l /n_0\rfloor$ are zero, and otherwise it is larger. 
This completes the proof. 
\end{proof}

\begin{figure} 
\centering
\includegraphics[clip= true, trim=7.2cm 19.6cm 6.8cm 2.7cm, width=.55\textwidth]{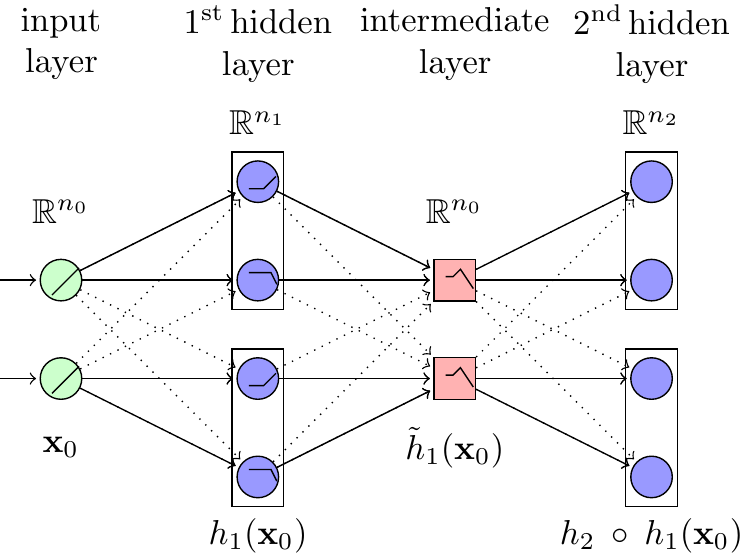} 
\includegraphics[width=.44\textwidth]{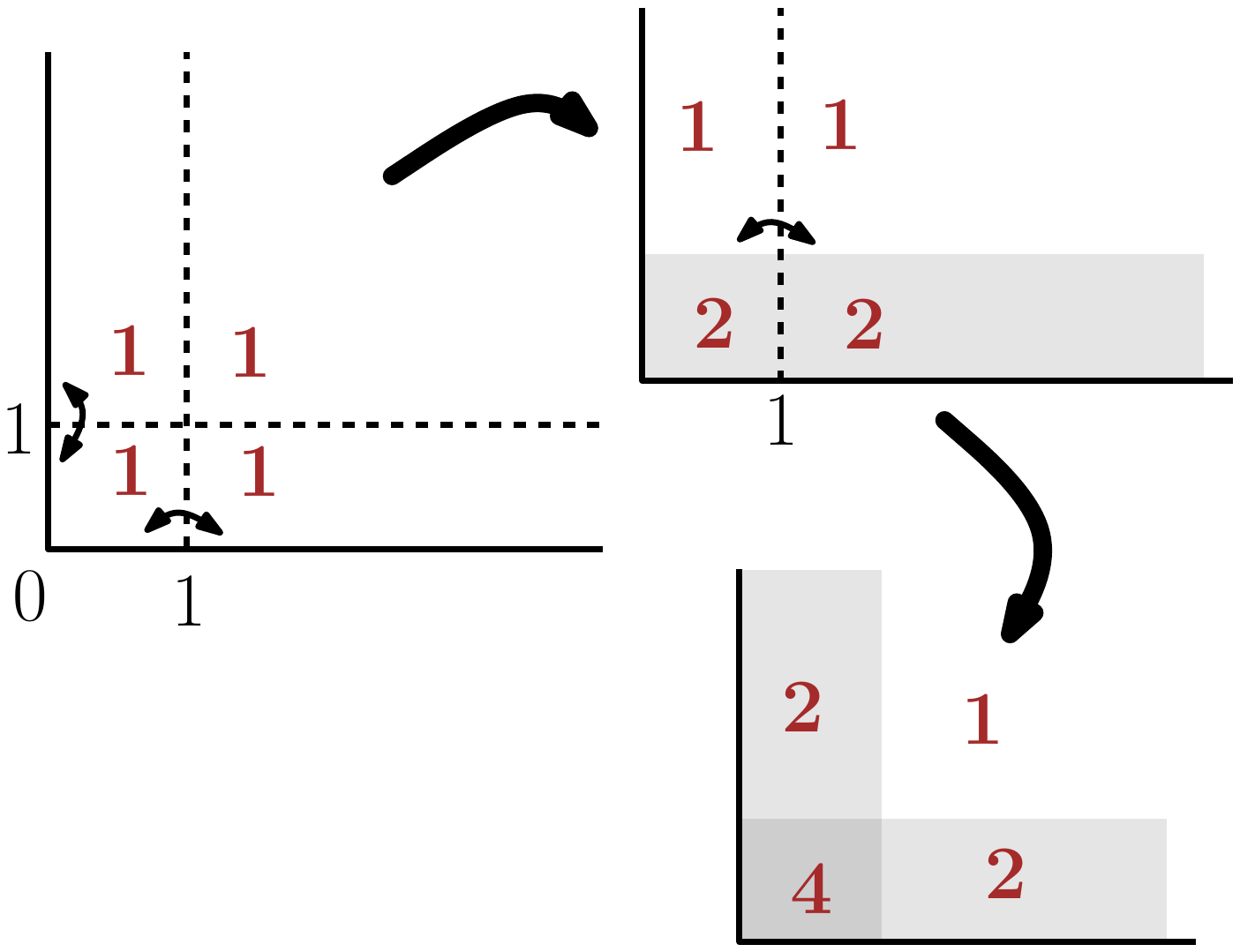} 
\caption{
Illustration of the proof of Theorem~\ref{thm:maxout}. 
The figure shows a rectifier network that has been divided into $n_0$ independent parts. Each part is sensitive only to one coordinate of the input-space. 
Each layer of each part is fed to a fictitious intermediary affine unit (computed by the preactivation function of the next layer), 
which computes the activation value that is passed to the next layer. 
Illustration of a function computed by the depicted rectifier network for $n_0=2$, at the intermediary layer. 
The function is composed of two foldings; the first pair of hidden units fold the input-space $\R^2$ along a line parallel to the x-axis, and the second pair, along a line parallel to the y-axis. 
}
\label{figure:maxout}
\end{figure}

\section{Our Bounds in terms of Parameters}
\label{sec:bound_params}

We computed bounds for the maximal number of linear regions of the functions computable by different networks in terms of their number of hidden units. 
It is not difficult to express these results in terms of the number of parameters of the networks and to derive expressions for the asymptotic rate of growth of the number of linear regions per added parameter. 
This kind of expansions have been computed in~\citet[][Proposition~8]{Pascanu2014}. 
The number of parameters of a deep model with $L$ layers of width $n$ behaves as $\Theta(L n^2)$, i.e., it is bounded above and below by $L n^2$, asymptotically. The number of parameters of a shallow model with $Ln$ hidden units behaves as $\Theta(L n)$. 
Our Theorem~\ref{thm:opt2} and the discussion of shallow networks given in Sec.~\ref{sec:single}, 
imply the following asymptotic rates (number of linear regions per parameter): 
\begin{itemize}
\item For a  deep model: $\Omega\left(\left(\slantfrac{n}{n_0}\right)^{n_0 (L-1)}\frac{n^{n_0-2}}{L}\right)$. 
\item For a shallow model: $O\left(L^{n_0-1} n^{n_0-1}\right)$. 
\end{itemize}
This shows that, for deep models, the maximal number of linear regions grows exponentially 
fast with the number of parameters, whereas, for shallow models, it grows only polynomially fast with the number of parameters.

\section{Maxout Networks}
\label{sec:maxout_apdx}

\begin{proof}[Proof of Proposition~\ref{prp:maxout}]
Here we investigate the maximal number of linear regions of a rank-$k$ maxout layer with $n$ inputs and $m$ outputs. 
In the case of rectifier units, the solutions is simply the maximal number of regions of a hyperplane arrangement. 
In the case of maxout units, we do not have hyperplane arrangements. 
However, we can upper bound the number of linear regions of a maxout layer by the number of regions of a hyperplane arrangement. 
The arguments are as follows. 

As mentioned in Sec.~\ref{sec:maxout}, each maxout unit divides its input into the linear regions of an upper envelope of $k$ real valued linear functions. 
In other words, the input space is divided by pieces of hyperplanes defining the boundaries between inputs where one entry of the pre-activation vector is larger than another. 
There are at most $k^2$ such boundaries, since each of them corresponds to the solution set of an equation of the form $f_i(\vx) = f_j(\vx)$. 
If we extend each such boundary to a hyperplane, then the number of regions can only go up. 

The linear regions of the layer are given by the intersections of the regions of the individual units. 
Hence, the number of linear regions of the layer is upper bounded (very loosely) by the number of regions of an arrangement of $k^2 \cdot m$ hyperplanes in $n$-dimensional space. 
By~\citet{zaslavsky1975facing}, the latter is $\sum_{j=0}^{n}{ k^2 m\choose j}$, 
which behaves as $O((k^2 m)^n)$, i.e., polynomially in $k$ and in $m$. 
\end{proof}

\begin{proof}[Proof of Theorem~\ref{thm:maxout}]
Consider a network with $n=n_0$ maxout units of rank $k$ in each layer. See Fig.~\ref{figure:maxout}. 
We define the seeds of the maxout unit $q_j$ such that $\{ \mW_{i,:}\}_i$ 
are unit vectors pointing in the positive and negative direction of $\lfloor k/2\rfloor$ coordinate vectors. 
If $k$ is larger than $2n_0$, then we forget about $k-2n_0$ of them (just choose $\mW_{i,:}=0$ for $i>2n_0$). 
In this case, $q_j$ is symmetric about the coordinate hyperplanes with normals $e_i$ with $i\leq \lfloor k/2\rfloor$ and has one linear region for each such $i$, 
with gradient $e_i$. 
For the remaining $q_j$ we consider similar functions, whereby we change the coordinate system by a slight rotation in some independent direction. 

This implies that the output of each $q_j\circ (f_1,\ldots, f_k)$ is an interval $[0,\infty)$. 
The linear regions of each such composition divide the input space into $r$ regions $R_{j,1},\ldots, R_{j,k}$. 
Since the change of coordinates used for each of them is a slight rotation in independent directions, 
we have that $R_i : =\cap_j R_{j,i}$ is a cone of dimension $n_0$ for all $i\in[k]$. 
Furthermore, the gradients of $q_j\circ f_j$ for $j\in[n_0]$ on each $R_i$ are a basis of $\R^{n_0}$. 
Hence the image of each $R_i$ by the maxout layer contains an open cone of $\R^{n_0}$ which is identical for all $i\in[k]$. 
This image can be shifted by bias terms such that the effective input of the next layer contains an open neighbourhood of the origin of $\R^{n_0}$. 

The above arguments show that a maxout layer of width $n_0$ and rank $k$ can identify at least $k$ regions of its input. 
A network with $L-1$ layers with therefore identify $k^{L-1}$ regions of the input.  
\end{proof}

\begin{figure}
\begin{center}
\includegraphics[clip=true, trim=7.2cm 19.6cm 6.8cm 2.7cm, width=.55\textwidth]{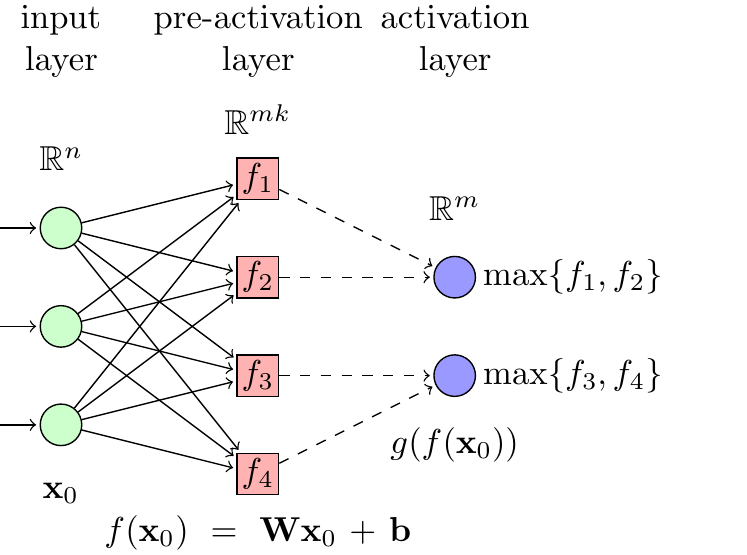}
\includegraphics[clip=true, trim=1.5cm 18.6cm 13cm 3.7cm, width=.44\textwidth]{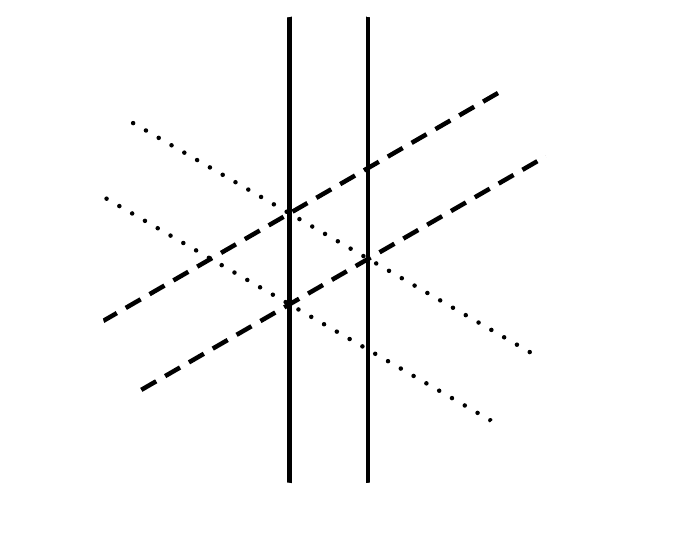}
\end{center}
\caption{Illustration of a rank-$2$ maxout layer with $n=3$ inputs and $m=2$ outputs. 
The preactivation function maps the input into $m k$-dimensional space, where $k=2$ is the rank of the layer. 
The activation function maximizes over groups of $k$ preactivation values.  
Illustration of the $3$-dimensional Shi arrangement $\mathcal{S}_3\subset \R^3$ 
(depicted is the intersection with $\{\vx \in \R^3\colon \vx_1+\vx_2+\vx_3=1\}$). 
This arrangement corresponds an input-space partition of a rank-$3$ maxout layer with $n=3$ inputs and $m=3$ outputs (for one choice of the parameters). 
Each pair of parallel hyperplanes delimits the linear regions of one maxout unit. }
\label{figure:shi}
\end{figure}

In Sec.~\ref{sec:maxout} we mentioned that maxout layers can compute functions whose linear regions 
correspond to intersections of Voronoi diagrams. 
Describing intersections of Voronoi diagrams is difficult, in general. 
There are some superpositions of Voronoi diagrams that correspond to hyperplane arrangements which are well understood. 
Here are two particularly nice examples: 

\begin{example}
Consider a layer with $n$ inputs and $m=n(n-1)/2$ rank-3 maxout units labeled by the pairs $(i,j)$, $1\leq i<j\leq n$. 
The input and bias weights can be chosen in such a way that the regions of unit $(i,j)$ are delimited by the two hyperplanes $H_{(i,j),s}= \{\vx\in\R^n\colon  \vx_i-\vx_j=s\}$ for $s\in\{0,1\}$. 
The intersections of the regions of all units are given by the regions of the hyperplane arrangement 
$\{H_{(i,j),s}\}_{1\leq i<j\leq n, s=1,2}$, which is known as the {\em Shi arrangement} $\mathcal{S}_n$ and has $(n+1)^{n-1}$ regions. 
The right panel of Fig.~\ref{figure:shi} illustrates the Shi arrangement $\Scal_3$. 

A related arrangement, corresponding to rank-4 maxout units, is the {\em Catalan arrangement}, which has triplets of
parallel hyperplanes, and a total of $n! C_n$ regions, where $C_n:=\frac{1}{n+1}{2n\choose
n}$ is the {\em Catalan number}.  For details on these arrangements
see~\cite[Corollary~5.1 and Proposition~5.15]{Stanley04}. 
\end{example}

\section{Other Networks}

In the introduction we mention that our analysis of rectifier and maxout networks serves as a platform to study other types of feedforward neural networks. 
Without going into many details, we exemplify this for the particular case of convolutional networks. 
A convolutional network is a network whose units take values in a space of {\em features} (real valued arrays) and whose edges pass features by convolution with {\em filters} (real valued arrays). 
Since convolution is a linear map, the preactivation function of a convolutional network is of the same form as the preactivation functions considered in this paper. Its output is a feature, but it can be written as a vector, like the $f_{l,i}$'s that we considered. 
Hence convolutional networks with piecewise linear activations fall in the class of networks that we considered here. 
The only difference lies in that the corresponding input weight matrices of convolutional networks belong to restricted classes of matrices. 

\section{Sinusoidal Boundary Experiment}
\label{sec:sinusoidal}

Here, we describe the experiments we performed to obtain
Fig.~\ref{fig:eye_candy} in the main text. 

In this experiment we considered two MLPs, of which one has a
single hidden layer with 20 hidden units and the other has two
hidden layers with 10 hidden unit each. The MLPs were trained on
the same synthetic dataset using a conjugate natural
gradient~\citep{Pascanu+Bengio-ICLR2014} which was used to
minimize the effect of optimization. We plot the best of several
runs. The shallow model misclassified 123 examples, whereas the
deep model did only 24 examples.  The two-layer model is better
at capturing a sinusoidal decision boundary, because it can
define more linear regions. 

\section{Visualizing the Behaviour of Hidden Units in Higher Layers}
\label{sec:visualization}

In this section, we describe the details on how one can visualize
the effect of folding in rectifier MLPs, discussed in
Sec.~\ref{sec:faces}.

Any piecewise linear function is fully defined by the different
linear pieces from which it is composed. Each piece is given by
its domain--a region of the input space $R_i \subseteq
\R^{n_0}$--and the linear map $f_i$ that describes its behaviour
on $R_i$. Because $f_i$ is an affine map, it can be interpreted
in the same way hidden units in a shallow model are. Namely, we
can write $f_i$ as:
\[
    f_i(\vx) = \mathbf{u}^\top \vx + c,~~~ \vx \in R_i,
\]
where $\mathbf{u}^\top $ is a row vector, and $\mathbf{u}^\top  \in
\R^{n_0}$. Then $f_i$ measures the (unnormalized) cosine distance
between $\vx$ and $\mathbf{u}^\top$.  If $\vx$ is some image,
$\mathbf{u}^\top$ is also an image and shows the pattern
(template) to which the unit responds whenever $\vx \in R_i$.

Given an input example $\vx$ from an arbitrary region $R_i$ of the
input space
we \emph{can construct the corresponding linear map} $f_i$
generated by the $j$-th unit at the $l$-th layer. Specifically,
the weight $\mathbf{u}^\top$ of the linear map $f_i$ is computed by
\begin{equation}
\label{eq:reconstr}
\mathbf{u}^\top = 
\left({\mW_{l}}\right)_{j:}
\diag\left(\mathbf{I}_{f_{l-1}>0}\left(\vx\right)\right)
\mW_{l-1}
\cdots
\diag\left(\mathbf{I}_{f_{1}>0}\left(\vx\right)\right)
\mW_{1}.
\end{equation}
A bias of the linear map can be similarly computed.

From Eq.~\eqref{eq:reconstr}, we can see that the linear map of a
specific hidden unit $f_{l,j}$ at a layer $l$ is found by keeping
track of which linear piece is used at each layer until the layer
$l$ ($\mathbf{I}_{f_{p}>0},p<l$ \todor{-- which is the indicator function}). At the end, the $j$-th row of
the weight matrix $\mW_{l}$ ($\left({\mW_{l}}\right)_{j:}$) is
multiplied. Although we present a formula specific to the
rectifier MLP, it is straightforward to adapt this to any MLP
with a piecewise linear activation, such as a convolutional
neural network with a maxout activation.

From the fact that the linear map computed by
Eq.~\eqref{eq:reconstr} depends on each sample/point $\vx$, we
\todor{need to traverse a set of points (e.g., training samples) to
identify different linear responses of a hidden unit. 
While this does not give all possible responses, if the set of 
points is large enough, we can get sufficiently many to provide 
a better understanding of its behaviour. }

We trained a rectifier MLP with three hidden layer on Toronto
Faces Dataset (TFD)~\citep{Susskind2010}. The first two hidden
layers have 1000 hidden units each and the last one has 100
units. 

We trained the model using stochastic gradient descent. We used, as
regularization, an $L_2$ penalty with a coefficient of $10^{-3}$, dropout on
the first two hidden layers (with a drop probability of $0.5$) and we enforced
the weights to have unit norm column-wise by projecting the weights after each
SGD step. We used a learning rate of $0.1$ and the output layer is composed of
sigmoid units.
The purpose of these regularization schemes, and the sigmoid
output layer is to obtain cleaner and sharper filters. The model
is trained on fold 1 of the dataset and achieves an error of
20.49\% which is reasonable for this dataset and a
non-convolutional model. 

Since each of the first layer hidden units only responds to a
single linear region, we directly visualize the learned weight
vectors of the 16 randomly selected units in the first hidden
layer. These are shown on the top row of
Fig.~\ref{fig:filtersTFD}.

On the other hand, for any other hidden layer, we randomly pick
20 units per layer and visualize the most interesting four units
out of them based on the maximal Euclidean distance between the
different linear responses of each unit. \todor{The linear responses of
each unit are computed by clustering the responses obtained 
on the training set (we only consider those responses where the activation 
was positive) into four clusters using K-means algorithm. 
We show the representative linear response in each of the clusters (see the 
second and third rows of Fig.~\ref{fig:filtersTFD})}. 

Similarly, we visualize the linear maps learned by each of the
output unit. For the output layer, we show the visualization of
all seven units in Fig.~\ref{fig:outputTFD}

By looking at the differences among the distinct linear regions
that a hidden unit responds to, we can investigate the type of
invariance the unit learned. In Fig.~\ref{fig:diffTFD}, we show
the differences among the four linear maps learned by the last
visualized hidden unit of the third hidden layer (the last column of the
 visualized linear maps). 

From these visualizations, we can see that a hidden unit learns
to be invariant to more abstract and interesting translations at
higher layers. We also see the types of invariance of a hidden
unit in a higher layer clearly. 

\begin{figure}[ht]

    \begin{minipage}{0.15\textwidth}
        \centering
        \textcolor{white}{L1}
    \end{minipage}  
    \begin{minipage}{0.33\textwidth}
        \centering
        Linear maps
    \end{minipage}  
    \begin{minipage}{0.33\textwidth}
        \centering
        Linear maps \\
        (Normalized)
    \end{minipage}
    
    \begin{minipage}{0.15\textwidth}
        \centering
        L1
    \end{minipage}  
    \begin{minipage}{0.33\textwidth}
        \centering
        \includegraphics[width=0.99\columnwidth, clip=true, trim=4.5cm 1.5cm 4cm 1.5cm]{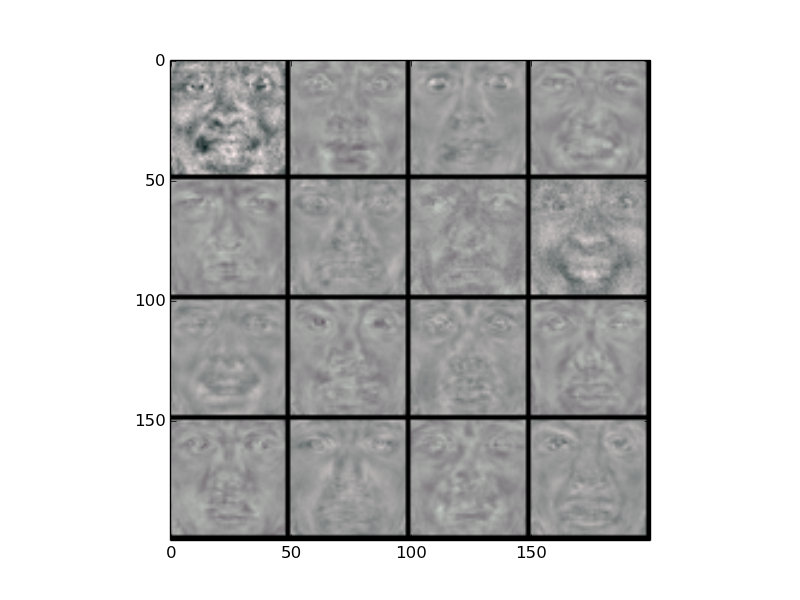}
    \end{minipage}  
    \begin{minipage}{0.33\textwidth}
        \centering
        \includegraphics[width=0.99\columnwidth, clip=true, trim=4.5cm 1.5cm 4cm 1.5cm]{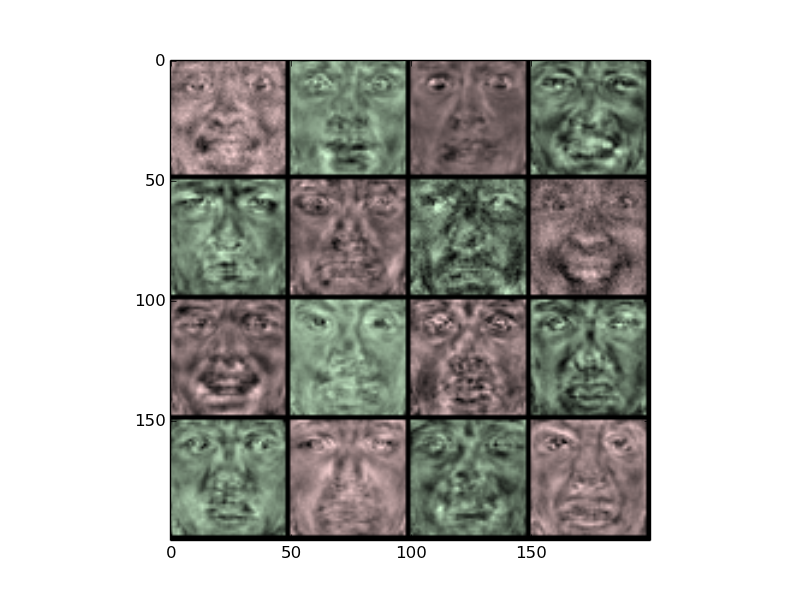}
    \end{minipage}

    \begin{minipage}{0.15\textwidth}
        \centering
        L2
    \end{minipage}  
    \begin{minipage}{0.33\textwidth}
        \centering
        \includegraphics[width=0.99\columnwidth, clip=true, trim=4.5cm 1.5cm 4cm 1.5cm]{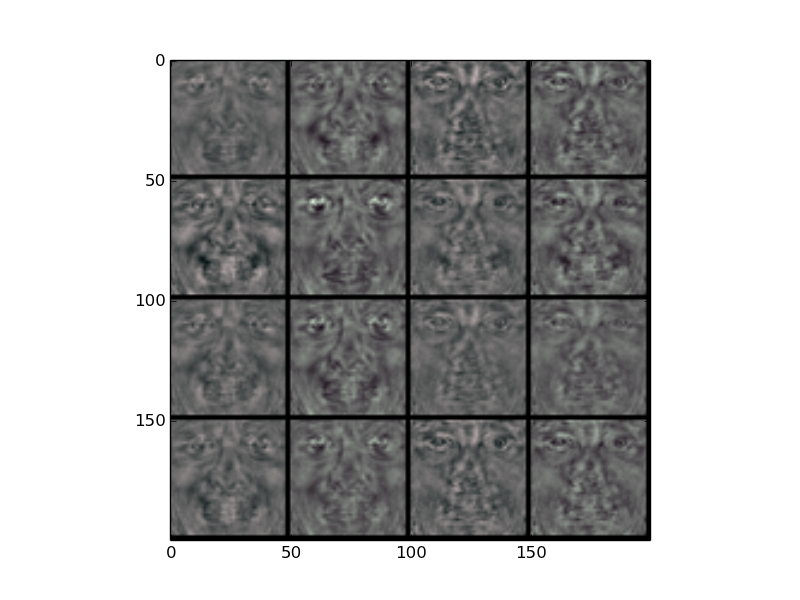}
    \end{minipage}
    \begin{minipage}{0.33\textwidth}
        \centering
        \includegraphics[width=0.99\columnwidth, clip=true, trim=4.5cm 1.5cm 4cm 1.5cm]{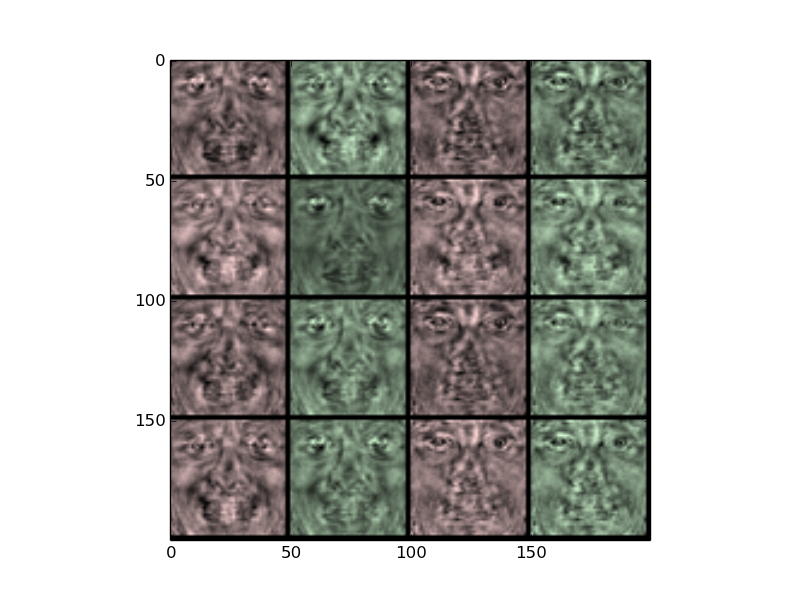}
    \end{minipage}
    
    \begin{minipage}{0.15\textwidth}
        \centering
        L3
    \end{minipage}  
    \begin{minipage}{0.33\textwidth}
        \centering
        \includegraphics[width=0.99\columnwidth, clip=true, trim=4.5cm 1.5cm 4cm 1.5cm]{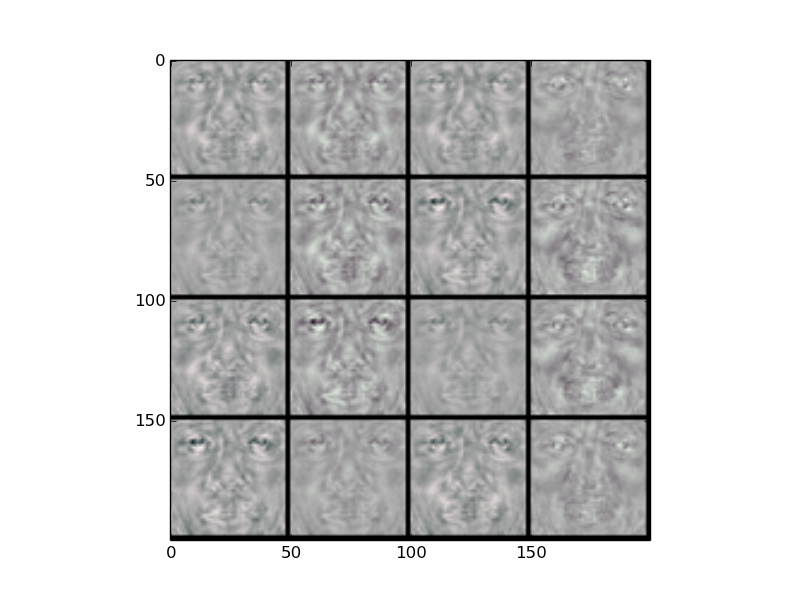}
    \end{minipage}
    \begin{minipage}{0.33\textwidth}
        \centering
        \includegraphics[width=0.99\columnwidth, clip=true, trim=4.5cm 1.5cm 4cm 1.5cm]{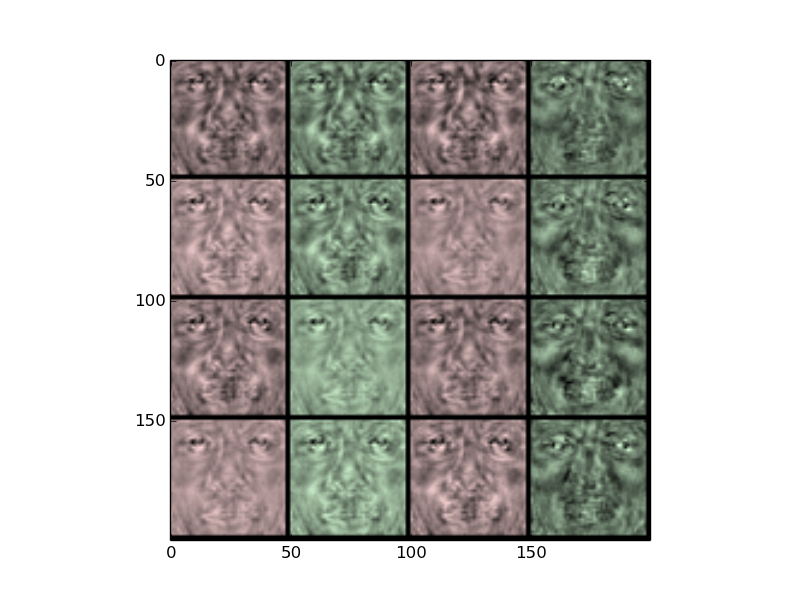}
    \end{minipage}
    
    \caption{
        Visualizations of the linear maps learned by each hidden
            layer of a rectifier MLP trained on TFD dataset. Each
            row corresponds to each hidden layer. The first
            column shows the unnormalized linear maps, and the
            last column shows the normalized linear maps showing
            only the direction of each map. Colors are only used
            to improve the distinction among different filters.
    }
\label{fig:filtersTFD}
\end{figure}

\begin{figure}[th]
    \centering
    \hfill
    \begin{minipage}{0.24\textwidth}
        \centering
        \includegraphics[width=0.99\columnwidth, clip=true, trim=4.5cm 1.5cm 4cm 1.5cm]{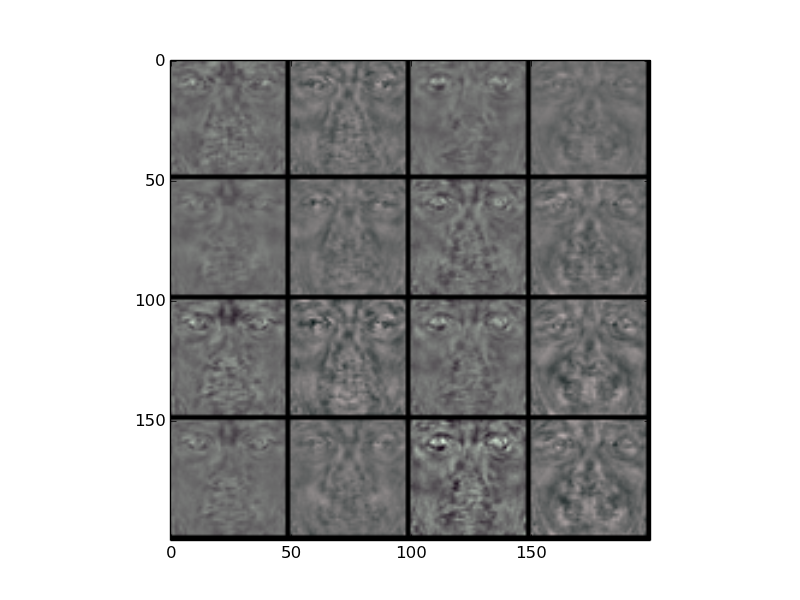}
    \end{minipage}
    \begin{minipage}{0.24\textwidth}
        \centering
        \includegraphics[width=0.99\columnwidth, clip=true, trim=4.5cm 1.5cm 4cm 1.5cm]{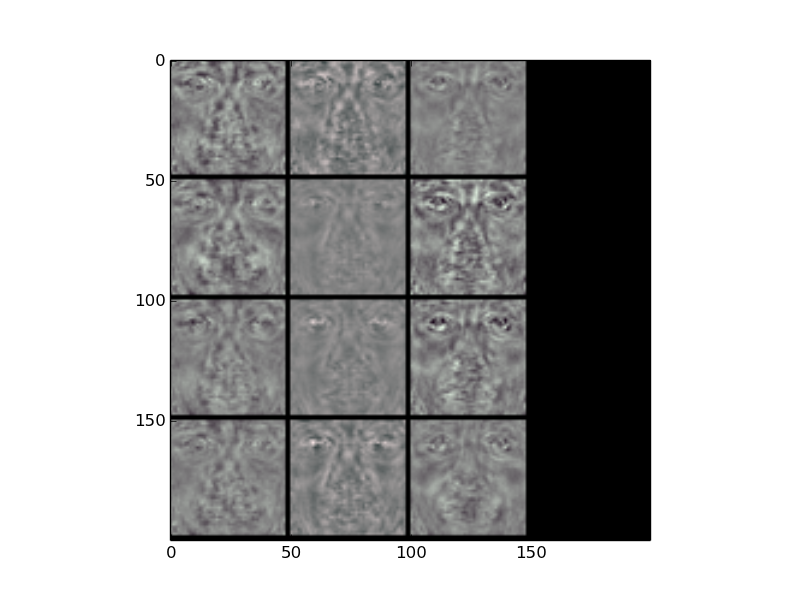}
    \end{minipage}
    \hfill
    \begin{minipage}{0.24\textwidth}
        \centering
        \includegraphics[width=0.99\columnwidth, clip=true, trim=4.5cm 1.5cm 4cm 1.5cm]{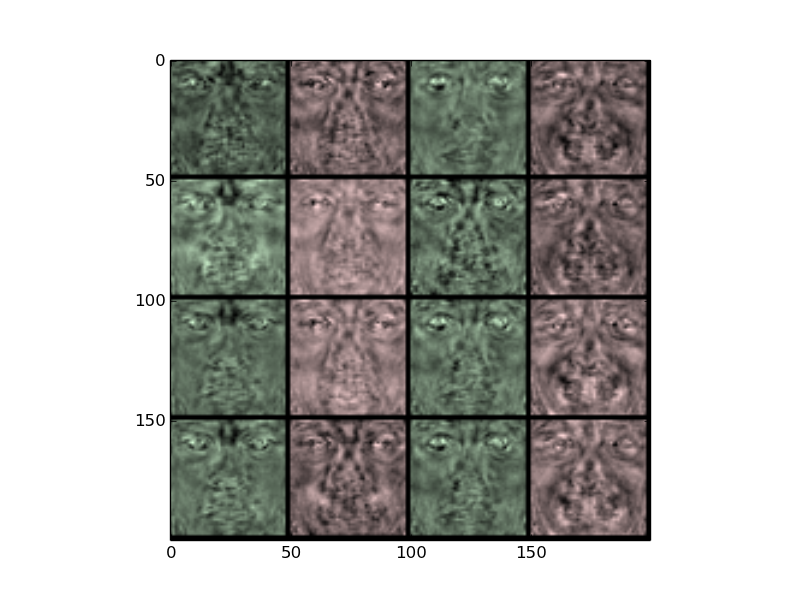}
    \end{minipage}
    \begin{minipage}{0.24\textwidth}
        \centering
        \includegraphics[width=0.99\columnwidth, clip=true, trim=4.5cm 1.5cm 4cm 1.5cm]{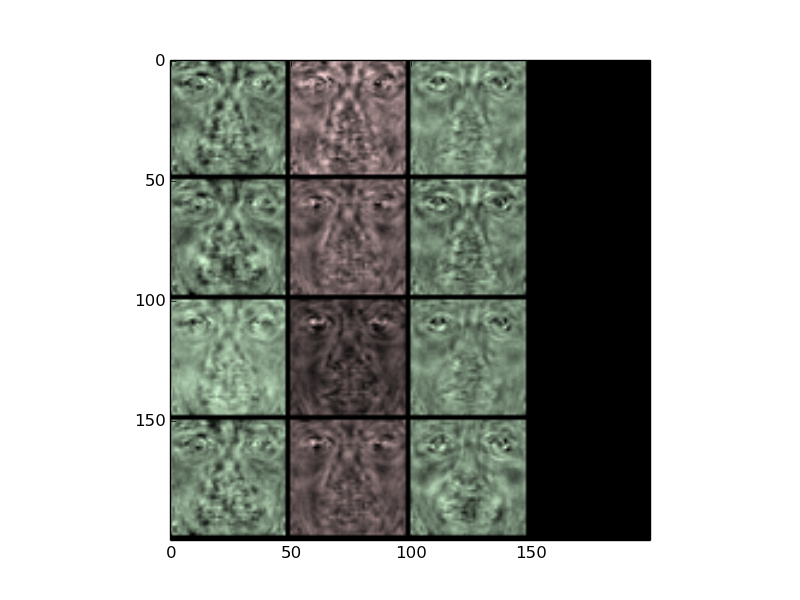}
    \end{minipage}
    \hfill

    \begin{minipage}{0.48\textwidth}
    \centering
    Linear maps
    \end{minipage}
    \hfill
    \begin{minipage}{0.48\textwidth}
    \centering
    Normalized Linear Maps
    \end{minipage}
    
\caption{
    Linear maps of the output units of the rectifier MLP trained
        on TFD dataset. The corresponding class labels for the
        columns are (1)
        anger, (2) disgust, (3) fear, (4) happy, (5) sad, (6)
        surprise and (7) neutral.
}
\label{fig:outputTFD}
\end{figure}

\begin{figure}[ht]
    \centering
    \begin{minipage}{0.33\textwidth}
        \centering
        \includegraphics[width=0.99\columnwidth, clip=true, trim=4.5cm 1.5cm 4cm 1.5cm]{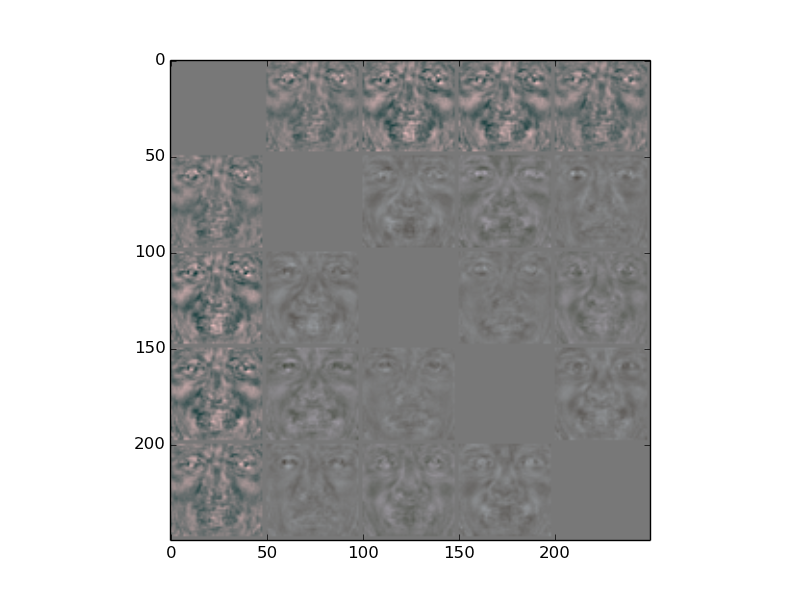}
    \end{minipage}
    \begin{minipage}{0.33\textwidth}
        \centering
        \includegraphics[width=0.99\columnwidth, clip=true, trim=4.5cm 1.5cm 4cm 1.5cm]{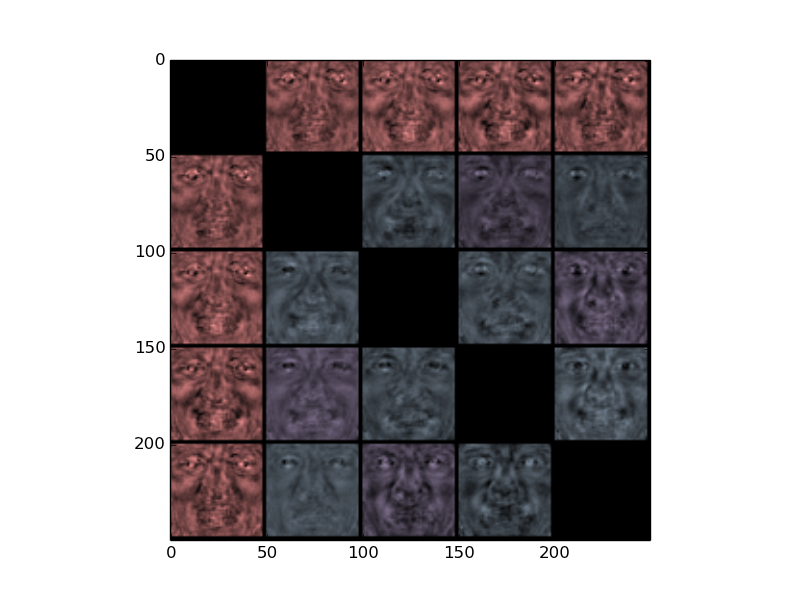}
    \end{minipage}

    \begin{minipage}{0.33\textwidth}
        \centering
        Difference Matrix
    \end{minipage}
    \begin{minipage}{0.33\textwidth}
        \centering
        Normalized Difference Matrix
    \end{minipage}

    \caption{Differences among the distinct linear regions of a
        single hidden unit at the third hidden layer of the
            rectifier MLP trained on TFD.}
    \label{fig:diffTFD}
\end{figure}

\citet{Zeiler+et+al-arxiv2013b} attempt to visualize the
behaviour of units in the upper layer, specifically, 
of a deep convolutional network with rectifiers. This approach is
to some extent similar to our approach proposed here, except that
we do not make any other assumption beside that a hidden unit in
a networks uses a piece-wise linear activation function. 

The perspective from which the visualization is considered is
also different. \citet{Zeiler+et+al-arxiv2013b} approaches the
problem of visualization from the perspective of (approximately)
inverting the feedforward computation of the neural network,
whereas our approach is derived by identifying a set of linear
maps per hidden unit.

This difference leads to a number of minor differences in the
actual implementation. For instance,
       \citet{Zeiler+et+al-arxiv2013b} approximates the inverse
       of a rectifier by simply using another rectifier. On the
       other hand, we do not need to approximate the inverse of the
       rectifier. Rather, we try to identify regions in the input
       space that maps to the same activation.

\begin{figure}[ht]
    \centering 
        \includegraphics[width=0.45\columnwidth, clip=true, trim=1cm 1cm 1cm 1cm]{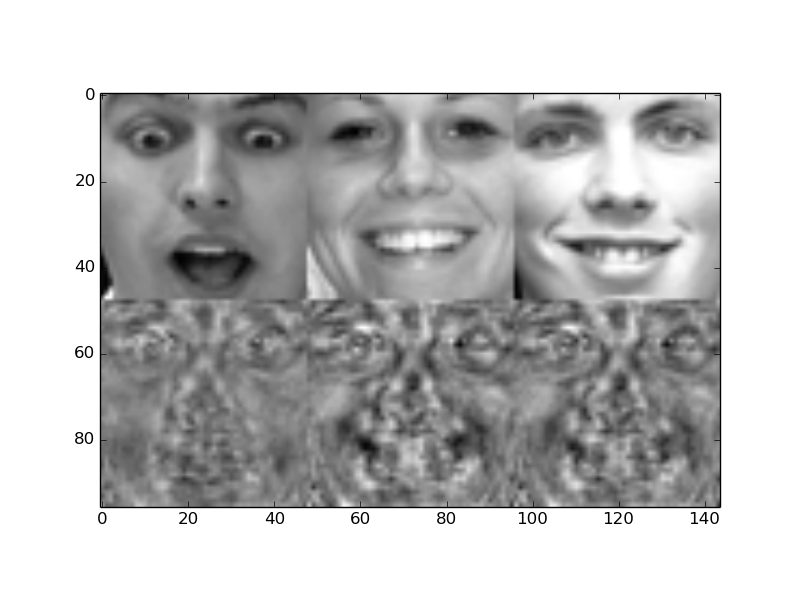}
\caption{The visualization of three distinct points in the input
    space that map to the same activation of a randomly chosen
        hidden unit at the third hidden layer. The top row shows
        three points ({\emph not} training/test samples) in the
        input space, and for each point, we plot the linear map
        below.}
\label{fig:corresponding}
\end{figure}

In our approach, it is possible to visualize an actual point in
the input space that maps to the same activation of a hidden
unit. In Fig.~\ref{fig:corresponding}, we show three distinct
points in the input space that activates a randomly chosen hidden
unit in the third hidden layer to be exactly $2.5$. We found
these points by first finding three training samples that map to
an activation close to $2.5$ of the same hidden unit, and from
each found sample, we search along the linear map (computed by
Eq.~\eqref{eq:reconstr}) for a point that exactly results in the
activation of $2.5$. Obviously, the found point is {\emph not}
one of the training samples. 
From those three points, we can see that the chosen hidden unit
responds to a face with wide-open mouth and a set of open eyes
while being invariant to other features of a face (e.g.,
eye brows). \todor{By the pertubation analysis, we can assume that there 
is an open set around each of these points that are identified 
by the hidden unit.}